\documentclass[11pt,letterpaper]{article}

\usepackage[T1]{fontenc}
\usepackage[utf8]{inputenc}
\usepackage{amsmath,amssymb,amsfonts,amsthm,bm,mathtools}
\usepackage{booktabs,longtable,array,multirow,graphicx,enumitem,microtype,url,float,placeins}
\usepackage{geometry}
\usepackage{natbib}
\usepackage{authblk}
\usepackage{algorithm}
\usepackage{algorithmic}
\usepackage[labelfont=bf,labelsep=colon]{caption}
\usepackage{hyperref}

\hypersetup{colorlinks=true,linkcolor=blue,citecolor=blue,urlcolor=blue}

\newcommand{\E}{\mathbb{E}}
\newcommand{\Pp}{\mathbb{P}}
\newcommand{\R}{\mathbb{R}}
\newcommand{\Pn}{\mathbb{P}_n}

\newcommand{\ind}{\mathbf{1}}
\newcommand{\Acal}{\mathcal{A}}
\newcommand{\Hcal}{\mathcal{H}}

\newcommand{\Qcal}{\mathcal{Q}}
\newcommand{\Gcal}{\mathcal{G}}

\newcommand{\Ucal}{\mathcal{U}}
\newcommand{\norm}[1]{\lVert #1\rVert}

\DeclareMathOperator{\logit}{logit}
\DeclareMathOperator{\Var}{Var}
\DeclareMathOperator*{\argmin}{arg\,min}

\theoremstyle{definition}
\newtheorem{assumption}{Assumption}
\newtheorem{drfdldefinition}{Definition}
\theoremstyle{plain}
\newtheorem{drfdllemma}{Lemma}
\newtheorem{drfdlproposition}{Proposition}
\newtheorem{drfdltheorem}{Theorem}
\newtheorem{drfdlcorollary}{Corollary}
\title{\textbf{Doubly Robust Functional Representation Learning for\\ Longitudinal Causal Inference with Irregular Histories}}

\author[1]{Mengfei Ran}
\author[1]{Yifeng Shen}
\author[2]{Ruijie Guan\footnote{Corresponding author. \url{guanrj@aircas.ac.cn}}}
\affil[1]{Academy of Life and Natural Sciences, Xi'an Jiaotong-Liverpool University, Suzhou, China}
\affil[2]{Aerospace Information Research Institute, Chinese Academy of Sciences, Beijing, China}
\date{}

\begin{document}

\maketitle

\begin{abstract}
Longitudinal causal studies often record histories as irregular functional fragments: laboratory values, physiologic signals, sensor streams, and image-derived summaries measured at unequal and informative times.  Standard doubly robust estimators usually require scalar summaries, whereas sequence learners optimize prediction losses that need not stabilize the efficient influence function. We propose Doubly Robust Functional Representation Learning (DR-FRL), a cross-fitted workflow that turns irregular histories into estimand-targeted states for observed-history regimes.  Functional and temporal encoders map point clouds and prior histories into states; nuisance heads estimate outcome, treatment, and censoring functions; and EIF-targeted validation, calibration, overlap, tail, and ablation diagnostics assess whether the state supports the estimating equation. If the selected state preserves the nuisance information needed by the EIF, representation error enters the same second-order product remainder as ordinary nuisance error, and the mean estimator is asymptotically linear under explicit rate, overlap, calibration, and stability conditions. Catoni aggregation is treated separately as a bounded-influence point estimator, not a replacement for Wald inference.  Simulations show gains when functional confounding is high-dimensional, measurement is informative, support is weak, or pseudo-outcomes are heavy-tailed.  A VitalDB audit shows that DR-FRL can use irregular laboratory point clouds and deliver a useful negative finding: for this ICU-disposition endpoint, scalar laboratory summaries already carry much endpoint-relevant information.
\end{abstract}

\vspace{0.1in}

\section{Introduction}

Consider a perioperative study in which clinicians decide whether and when to administer vasopressor support while a patient's risk state evolves.  Before the decision, the record may contain a few laboratory values measured days earlier, multiple vital-sign summaries recorded at unequal times, surgery descriptors, prior medications, and missingness patterns that themselves reflect clinical concern.  A scalar baseline table can hide much of this structure, but a black-box prediction model does not by itself answer the causal question of how an intervention regime would change postoperative risk.  The statistical problem is to use irregular functional histories without losing the estimating-equation structure that makes doubly robust causal inference interpretable.

The classical longitudinal causal framework is clear once a suitable history has been specified.  The g-formula of \cite{robins1986new} and marginal structural models \citep{robins2000marginal} identify effects under sequential exchangeability, consistency, and positivity.  Doubly robust and targeted-learning estimators combine outcome and mechanism models so that first-order bias is reduced when one nuisance component is well estimated \citep{bang2005doubly, vanderlaan2006targeted}.  \cite{bickel1993efficient} developed the semiparametric efficiency theory that clarifies the role of orthogonal scores, and modern cross-fitting makes it possible to use flexible nuisance learners without classical Donsker restrictions \citep{chernozhukov2018dml}.  These tools are now routine in longitudinal causal analysis, but they usually start after the analyst has already compressed the observed history into a finite vector.

That compression step is increasingly decisive.  Electronic health records, perioperative monitoring systems, wearable devices, imaging pipelines, and mobile-health platforms often record evolving histories as noisy, sparse, and irregular functional measurements.  Functional data analysis provides a rich language for such objects; \cite{ramsay2005functional} give the foundational treatment, and \cite{yao2005functional} developed a widely used framework for sparse longitudinal functional data.  Functional principal components, spline summaries, and functional linear models can be highly effective for regression and prediction \citep{cardot1999functional, james2000principal}, while nonparametric functional methods offer alternatives when smoothness assumptions are credible \citep{ferraty2006nonparametric}.  In causal inference, however, the relevant dimension of a curve is not its marginal variance or reconstruction error, but its role in the efficient influence function: a low-variance feature can still drive treatment, censoring, or the outcome regression.  Existing functional causal work has made important progress, but much of it focuses on functional treatments or fixed functional summaries rather than longitudinal treatment-confounder feedback with learned states \citep{ciarleglio2015treatment}.

Deep longitudinal models address another part of the problem.  \cite{lim2018rmsn} and \cite{bica2020crn} showed that learned sequence representations can improve counterfactual prediction, and later recurrent or attention-based architectures further improved dynamic treatment-response modeling \citep{li2021gnet, melnychuk2022causaltransformer}.  The attention mechanism of \cite{vaswani2017attention} is a powerful representation tool, but an attention layer is not an estimating equation.  Prediction-oriented sequence models can be useful comparators, yet their objectives do not directly measure whether a learned state stabilizes inverse-probability weights, calibrates treatment and censoring probabilities, or makes the efficient influence function nearly mean zero.  This leaves a gap between flexible representation learning and semiparametric causal inference.

This paper fills that gap for observed-history causal effects with irregular functional covariates.  At each decision time, the analyst observes scalar variables and a point cloud such as laboratory values, blood-pressure measurements, sensor signals, or image-derived summaries.  We do not require these point clouds to be reconstructed on a regular grid.  Measurement times, counts, marker identities, and noisy values are part of the observed pre-treatment history.  The estimand is deliberately an observed-history value under exchangeability conditional on the actually observed point clouds and sampling patterns.  A latent full-functional value would require additional coarsening or measurement assumptions, so it is not folded into the main identification claim.

We propose Doubly Robust Functional Representation Learning (DR-FRL).  A permutation-invariant or continuous-time encoder maps each irregular point cloud to a visit-level embedding; a temporal encoder maps previous embeddings, scalar covariates, treatments, follow-up indicators, and timing features into a learned state; nuisance heads estimate outcome regressions, treatment probabilities, and censoring probabilities; and the fitted nuisance functions enter a sequentially doubly robust recursion.  The representation is not judged by reconstruction or prediction alone.  It is selected and audited by its ability to support the efficient influence function through balance, calibration, overlap, effective sample size, and pseudo-outcome stability.

The main theoretical result is intentionally conditional.  If the learned state preserves the nuisance information needed by the efficient influence function, representation error enters through the same second-order product remainder as ordinary nuisance error.  Under explicit rate, calibration, overlap, and stability conditions, the mean-aggregation estimator is asymptotically linear and admits Wald inference.  Robust Catoni aggregation is kept separate: it is a bounded-influence point-estimation device for unstable pseudo-outcomes, not a source of ordinary Wald intervals.

The contributions are fourfold.  First, we define an observed-history target for irregular functional measurements and keep it separate from latent full-functional targets that require extra measurement assumptions.  Second, we place learned functional states inside a sequential doubly robust estimating equation rather than treating them as preprocessing.  Third, we derive an expansion in which representation error, nuisance error, and clipping error are separated; the mean estimator has Wald inference under explicit product-rate, overlap, and stability conditions.  Fourth, simulations and a VitalDB audit show where functional states, robust aggregation, and overlap diagnostics add value, and where scalar summaries already suffice.

The rest of the paper is organized as follows.  Section~\ref{sec:framework} defines the observed-history estimand, identification assumptions, and efficient influence function.  Section~\ref{sec:method} describes the learned state, nuisance models, targeted validation, and diagnostic output.  Section~\ref{sec:algorithm} gives the cross-fitted implementation.  Section~\ref{sec:theory} states the main large-sample results and their scope.  Section~\ref{sec:simulation} reports reproducible simulations designed to stress representation, overlap, and tail behavior.  Section~\ref{sec:realdata} applies the method as a diagnostic audit in VitalDB.  Section~\ref{sec:discussion} summarizes what the method adds and where it remains limited.

\section{Framework and Identification}\label{sec:framework}

\subsection{Observed longitudinal functional data}

We observe independent units $O_1,\ldots,O_n\sim P_0$.  Each unit is followed over decision times $t=0,\ldots,T$.  Let $B\in\mathcal B$ denote baseline covariates.  At time $t$, before treatment assignment, we observe scalar covariates $L_t$ and a functional point cloud
\[
O_t^X=\{(U_{t\ell},\widetilde X_{t\ell}):\ell=1,\ldots,M_t\},\qquad U_{t\ell}\in\Ucal,
\]
where $M_t$ is the number of measurements and $\widetilde X_{t\ell}$ is a noisy evaluation of an underlying trajectory or signal.  The case $M_t=0$ is allowed.  Treatment $A_t\in\Acal$ is assigned after $(B,L_t,O_t^X)$ is available.  Follow-up indicator $R_t\in\{0,1\}$ is then recorded, with $R_t=1$ meaning that the unit remains under follow-up after time $t$.  Define $\bar R_t=\prod_{s=0}^t R_s$ and $\bar A_t=(A_0,\ldots,A_t)$.

The pre-treatment observed history is
\[
H_t=(B,O_0^X,L_0,A_0,R_0,\ldots,O_{t-1}^X,L_{t-1},A_{t-1},R_{t-1},O_t^X,L_t),
\]
with $H_0=(B,O_0^X,L_0)$.  Only variables temporally prior to $A_t$ are included in $H_t$.  This ordering is essential in applications such as critical care, where a physiological measurement after a treatment change may be a mediator rather than a confounder.

A nonparametric observed-data law can be factorized as
\begin{align}
 p_0(o)=p_0(b)
 \prod_{t=0}^T &p_0(m_t,u_{t1:m_t},\widetilde x_{t1:m_t},l_t\mid h_t^{-})
 p_0(a_t\mid h_t,\bar r_{t-1}=1)\nonumber\\
 &\times p_0(r_t\mid h_t,a_t,\bar r_{t-1}=1)
 p_0(y\mid h_{T+1},\bar a_T,\bar r_T=1),
 \label{eq:factorization}
\end{align}
where $h_t^{-}$ is the history just before the new pre-treatment measurements at time $t$.  The first factor includes the functional observation process.  Equation~\eqref{eq:factorization} clarifies that the functional observation process, treatment mechanism, censoring mechanism, and outcome regression play distinct roles, following the same temporal ordering logic used by \cite{robins1986new} and in the longitudinal causal framework of \cite{hernan2020causal}.

\subsection{Regimes and nuisance functions}

A deterministic dynamic treatment regime is $g=\{g_t\}_{t=0}^T$ with $g_t:\Hcal_t\to\Acal$.  A static regime $\bar a=(a_0,\ldots,a_T)$ is the special case $g_t(h)=a_t$.  Such regimes are standard objects in sequential decision and dynamic-treatment-regime analyses; see, for example, \cite{robins2008estimation}.  The main theorems are written for deterministic static or dynamic regimes because this is the setting used in the longitudinal simulations.  Stochastic interventions $q=\{q_t(\cdot\mid h_t)\}_{t=0}^T$ are handled by the same recursion after replacing treatment indicators by regime density ratios; Supplementary Section S2 gives the corresponding notation.  We write $Y^g$ for the deterministic-regime potential outcome and target the value
\[
\theta(g)=\E(Y^g).
\]

For deterministic $g$, define the treatment and incremental censoring mechanisms
\begin{align*}
\pi_t^g(h)&=P_0\{A_t=g_t(h)\mid H_t=h,\bar R_{t-1}=1\},\\
\rho_t^g(h)&=P_0\{R_t=1\mid H_t=h,A_t=g_t(h),\bar R_{t-1}=1\}.
\end{align*}
The censoring model is incremental and conditional on current treatment.  This replaces the ambiguous cumulative notation $P(C_t=1\mid H_t)$ and is needed whenever follow-up can depend on the current treatment.

Let
\begin{equation}
G_t^g(H_t)=\prod_{s=0}^t \pi_s^g(H_s)\rho_s^g(H_s),\qquad
\mathcal H_t^g(O)=\frac{\ind\{\bar A_t=\bar g_t(\bar H_t),\bar R_t=1\}}{G_t^g(H_t)}.
\label{eq:cumulative-clever}
\end{equation}
Here $\mathcal H_t^g$ is the cumulative clever covariate.  Also define the incremental clever covariate
\begin{equation}
\omega_t^g(O)=\frac{\ind\{A_t=g_t(H_t),R_t=1\}}{\pi_t^g(H_t)\rho_t^g(H_t)},
\end{equation}
so that $\mathcal H_t^g=\prod_{s=0}^t\omega_s^g$.

To remove a possible ambiguity caused by censoring, $H_{t+1}$ in the next display means the next \emph{pre-treatment} history generated after $(H_t,A_t,R_t)$ on the at-risk path.  When $R_t=0$, the next pre-treatment history is not needed for the estimating equation; formally we extend the sample space with a cemetery value $\partial$ and set $H_{t+1}=\partial$ and $Q_{t+1}^g(\partial)$ arbitrary.  The factor $\ind(R_t=1)$ in $\omega_t^g$ and $\mathcal H_t^g$ makes that arbitrary value irrelevant.  Consequently, every residual
\[
M_t^{g,\star}=Q_{t+1}^{g,\star}(H_{t+1})-Q_t^{g,\star}(H_t)
\]
has conditional mean zero under the transition law given $(H_t,A_t=g_t(H_t),R_t=1,\bar R_{t-1}=1)$, not under the marginal law of $H_{t+1}$.

Set $Q_{T+1}^g(H_{T+1})=Y$.  For $t=T,\ldots,0$, define
\begin{equation}
Q_t^g(h)=\E_0\{Q_{t+1}^g(H_{t+1})\mid H_t=h,A_t=g_t(h),R_t=1,\bar R_{t-1}=1\}.
\label{eq:q-recursion}
\end{equation}
For a static regime, replace $g_t(h)$ by $a_t$.  Under the identification conditions below,
\begin{equation}
\theta(g)=\E_0\{Q_0^g(H_0)\}.
\label{eq:gformula}
\end{equation}

\subsection{Identification assumptions}

The observed-history target relies on the following assumptions, stated in potential-outcome language.  This notation follows \cite{rubin2005causal}; the same identification logic can also be expressed using graphical causal models \citep{pearl2009causality}.

\begin{assumption}[Consistency and no interference]\label{ass:consistency}
If a unit follows regime $g$ through time $T$ and remains observable according to the target follow-up intervention, then the observed outcome equals the corresponding potential outcome.  Interventions are well defined and there is no interference.
\end{assumption}

\begin{assumption}[Sequential exchangeability]\label{ass:exchangeability}
For each $t$ and regime $g$, let
\[
\mathcal F_t^g=(Y^g,H_{t+1}^g,\ldots,H_{T+1}^g)
\]
denote the counterfactual future process that would be generated after time $t$ under regime $g$, including the terminal outcome and future pre-treatment histories.  Sequential exchangeability requires
\[
\mathcal F_t^g \perp A_t\mid H_t,\bar R_{t-1}=1,\qquad
\mathcal F_t^g \perp R_t\mid H_t,A_t,\bar R_{t-1}=1.
\]
Thus treatment and follow-up at time $t$ are independent not only of the terminal counterfactual outcome but also of the future covariate and functional-history process needed to define later regime decisions.  For static regimes this reduces to the usual sequential ignorability condition for the outcome process.  The second statement treats censoring as loss of outcome follow-up or loss of required post-time-$t$ data.  Death is not censoring when death is the outcome.
\end{assumption}

\begin{assumption}[Positivity]\label{ass:positivity}
There exists $c>0$ such that, almost surely over histories relevant under $g$,
\[
\pi_t^g(H_t)\ge c,\qquad \rho_t^g(H_t)\ge c,\qquad t=0,\ldots,T.
\]
For stochastic-regime extensions, the density ratio $q_t(A_t\mid H_t)/\pi_t(A_t\mid H_t)$ is bounded on the support of $q_t$.
\end{assumption}

\begin{assumption}[Observed functional history]\label{ass:observed-history}
The observed point clouds, measurement counts and locations, scalar covariates, and missingness or measurement-intensity indicators needed to control confounding are included in $H_t$.  No reconstruction of a full latent function is required for the observed-history target.
\end{assumption}

If the scientific target instead conditions on a latent full trajectory $X_t(\cdot)$, the present observed-data likelihood is not enough without an additional bridge between latent and observed histories.  We therefore treat the observed-history value as the main estimand.  A latent full-functional value would require a separately specified measurement or coarsening model, for example a known design mechanism, validation data, replicate measurements, or a sensitivity model linking $H_t^{\mathrm{full}}$ to the observed point cloud $H_t^{\mathrm{obs}}$.  Because the present estimator does not include an additional inverse-coarsening or measurement-model step, no theorem below claims identification or efficiency for an unrestricted latent full-functional target.

\subsection{EIF and doubly robust recursion}

Under Assumptions~\ref{ass:consistency}-\ref{ass:observed-history}, the nonparametric efficient influence function for $\theta(g)$ is
\begin{equation}
\phi_g^\star(O)=
\sum_{t=0}^T \mathcal H_t^{g,\star}(O)
\{Q_{t+1}^{g,\star}(H_{t+1})-Q_t^{g,\star}(H_t)\}
+Q_0^{g,\star}(H_0)-\theta(g).
\label{eq:eif}
\end{equation}
This expression includes both treatment and censoring through the cumulative clever covariate in~\eqref{eq:cumulative-clever}.

Given working nuisance functions $\eta=(Q,\pi,\rho)$, define the backward pseudo-outcome recursion
\begin{align}
\widetilde Y_{T+1}^{g}(\eta)&=Y,\nonumber\\
\widetilde Y_t^{g}(\eta)&=Q_t^g(H_t)+\omega_t^g(\eta)
\{\widetilde Y_{t+1}^{g}(\eta)-Q_t^g(H_t)\},\qquad t=T,\ldots,0.
\label{eq:pseudo-recursion}
\end{align}
Algebraic telescoping gives
\begin{equation}
\widetilde Y_0^{g}(\eta)-\theta=
Q_0^g(H_0)-\theta+
\sum_{t=0}^T \mathcal H_t^g(\eta)\{Q_{t+1}^g(H_{t+1})-Q_t^g(H_t)\},
\label{eq:telescoping}
\end{equation}
with $Q_{T+1}^g=Y$.  Thus the recursion based on single-step weights and the EIF expansion based on cumulative weights are the same object written at different stages.
For later notation, write $D_g(O;\eta,\theta)$ for the right-hand side of \eqref{eq:telescoping}.  At the truth, $D_g(O;\eta^\star,\theta(g))=\phi_g^\star(O)$.

\section{Estimation and Diagnostics}\label{sec:method}

DR-FRL has three parts.  First, functional and temporal encoders turn the observed history into a state $S_t$.  Second, nuisance heads estimate the outcome, treatment, and censoring functions used by the EIF recursion.  Third, calibration, overlap, effective sample size, and pseudo-outcome-tail diagnostics show whether the fitted estimating equation is empirically stable.  The goal is not curve reconstruction; it is a state that preserves the nuisance information needed for the causal parameter.

\begin{figure}[t]
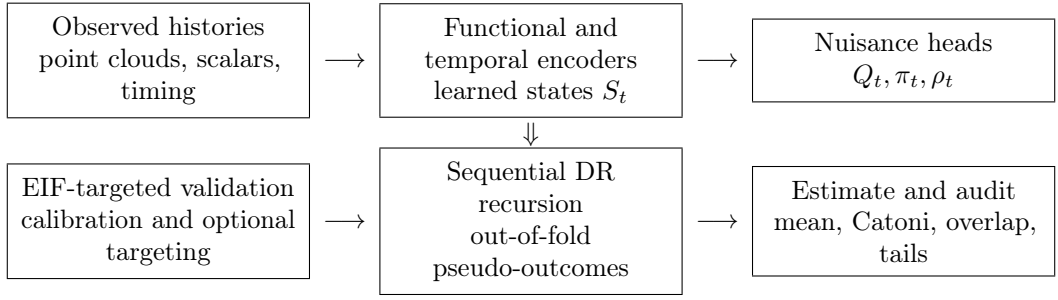

\centering
\small
\setlength{\fboxsep}{5pt}
\begin{tabular}{c@{\hspace{0.5em}}c@{\hspace{0.5em}}c@{\hspace{0.5em}}c@{\hspace{0.5em}}c}
\fbox{\parbox{0.22\textwidth}{\centering Observed histories\\point clouds, scalars, timing}}
& $\longrightarrow$ &
\fbox{\parbox{0.22\textwidth}{\centering Functional and temporal encoders\\learned states $S_t$}}
& $\longrightarrow$ &
\fbox{\parbox{0.22\textwidth}{\centering Nuisance heads\\$Q_t,\pi_t,\rho_t$}}\\[0.8em]
\multicolumn{5}{c}{$\Downarrow$}\\[-0.1em]
\fbox{\parbox{0.22\textwidth}{\centering EIF-targeted validation\\calibration and optional targeting}}
& $\longrightarrow$ &
\fbox{\parbox{0.22\textwidth}{\centering Sequential DR recursion\\out-of-fold pseudo-outcomes}}
& $\longrightarrow$ &
\fbox{\parbox{0.22\textwidth}{\centering Estimate and audit\\mean, Catoni, overlap, tails}}
\end{tabular}
\caption{DR-FRL workflow.  Functional histories are encoded into states, used in nuisance models, evaluated through cross-fitted pseudo-outcomes, and audited by calibration, overlap, effective sample size, and tail diagnostics.}
\label{fig:workflow}
\end{figure}

\subsection{State representations}

The functional encoder maps an irregular point cloud to a visit-level embedding,
\[
\xi_t=e_\phi^X(O_t^X,M_t)\in\R^{d_x}.
\]
The map is permutation invariant in the measurement pairs.  Candidate encoders include Deep Sets and Set Transformers for unordered point clouds \citep{zaheer2017deep, lee2019set}.  Spline or wavelet summaries provide simpler bases.  Continuous-time neural models handle irregular timing \citep{chen2018neural, rubanova2019latent}, while GRU-D and neural CDEs target missingness-aware or continuous-time sequence modeling \citep{che2018grud, kidger2020neuralcde}.  The temporal encoder maps the full pre-treatment history to a state
\[
S_t=e_\phi(H_t)=e_\phi^H(B,L_{0:t},\xi_{0:t},A_{0:t-1},R_{0:t-1},\Delta_{0:t},M_{0:t})\in\R^{d_s}.
\]
Here $\Delta_{0:t}$ contains visit-time gaps or within-visit timing features.  We use $S_t$ for the learned state to avoid conflict with scalar covariates.

\begin{drfdldefinition}[Causal sufficient representation]\label{def:suff}
A representation $S_t=e_\phi(H_t)$ is causally sufficient for $\theta(g)$ if there exist measurable functions $q_t^g$, $p_t^g$, and $r_t^g$ such that, almost surely,
\[
Q_t^{g,\star}(H_t)=q_t^g(S_t),\qquad
\pi_t^{g,\star}(H_t)=p_t^g(S_t),\qquad
\rho_t^{g,\star}(H_t)=r_t^g(S_t),
\]
for every $t=0,\ldots,T$.
\end{drfdldefinition}

For a representation class indexed by $\phi$ and nuisance model classes $\Qcal_t$ and $\Gcal_t$, define approximation errors
\begin{align*}
 b_{Q,t}(\phi)&=\inf_{q\in\Qcal_t}\norm{Q_t^{g,\star}(H_t)-q(e_\phi(H_t))}_{L_2(P_0)},\\
 b_{\Gamma,t}(\phi)&=\inf_{\gamma\in\Gcal_t}\norm{\Gamma_t^{g,\star}(H_t)-\gamma(e_\phi(H_t))}_{L_2(P_0)}.
\end{align*}
where $\Gamma_t^g=(\pi_t^g,\rho_t^g)$ or an equivalent cumulative-mechanism parameterization.  These quantities make explicit what is lost by replacing $H_t$ with $S_t$.

\subsection{Nuisance heads and calibration}

For each $t$, DR-FRL estimates
\[
\widehat Q_t^g(H_t)=\widehat q_t(S_t),\qquad
\widehat\pi_t^g(H_t)=\widehat p_t(S_t),\qquad
\widehat\rho_t^g(H_t)=\widehat r_t(S_t,g_t(H_t)).
\]
The censoring head is conditional on current treatment.  Treatment and censoring probabilities are calibrated within each training fold using temperature scaling, Platt scaling, or isotonic regression; \cite{platt1999probabilistic} gives the classical score-calibration construction and \cite{guo2017calibration} studies calibration behavior in modern neural networks.  When nonparametric calibration is preferred, the same diagnostic role can be served by score-to-probability transformations such as those studied by \cite{zadrozny2002transforming}.  For each time point we report Brier score, expected calibration error, minimum estimated probability, cumulative weight quantiles, and effective sample size
\[
\mathrm{ESS}_t=\frac{\left(\sum_i \widehat{\mathcal H}_{t,i}^g\right)^2}{\sum_i(\widehat{\mathcal H}_{t,i}^g)^2}.
\]

\subsection{EIF-targeted learning}\label{subsec:eifloss}

Prediction accuracy is not enough: a state that predicts outcomes well may still give unstable weights or a biased EIF mean.  Following the targeted-learning and orthogonal-score logic of \cite{vanderlaan2011targeted} and \cite{chernozhukov2018dml}, DR-FRL selects representations by how well they support the estimating equation.  Within each outer training fold, candidates are fitted on an inner fitting set, scored on an inner selection set, and evaluated only once on the outer held-out fold.  The validation objective is
\begin{align}
\mathcal L(\phi,\eta)=
&\sum_{t=0}^T\{\mathcal L_{Q,t}(\widehat Q_t)+\lambda_A\mathcal L_{A,t}(\widehat\pi_t)+\lambda_R\mathcal L_{R,t}(\widehat\rho_t)\}\nonumber\\
&+\lambda_{\mathrm{EIF}}\{\mathbb{P}_{n,\mathrm{val}}\widehat\phi_g(\phi,\eta)\}^2
+\lambda_{\mathrm{var}}\mathbb{P}_{n,\mathrm{val}}\{\widehat\phi_g(\phi,\eta)^2\}
+\lambda_{\mathrm{cal}}\mathcal L_{\mathrm{cal}}(\widehat\pi,\widehat\rho).
\label{eq:targeted-loss}
\end{align}
The first line fits the nuisance functions.  The second line penalizes EIF imbalance, volatile pseudo-outcomes, and poor mechanism calibration.  This criterion is a tuning rule, not an identification assumption.

Selection therefore does not prove that $b_{Q,t}(\widehat\phi)$ and $b_{\Gamma,t}(\widehat\phi)$ are small.  The asymptotic theorem remains conditional on the product-rate assumption.  The following finite-library result states only what independent validation can justify.

\begin{drfdlproposition}[Finite-library selection]\label{prop:selector}
Fix an outer training fold and a finite library $\Phi_n$ with $|\Phi_n|=M_n$.  For each $\phi\in\Phi_n$, fit nuisance heads on an inner fitting set and compute $\widehat R_{\mathrm{EIF}}(\phi)$ on an independent validation set, using the squared validation EIF mean plus the EIF variance penalty in \eqref{eq:targeted-loss}.  Let $R_{\mathrm{EIF}}(\phi)$ be the corresponding conditional population criterion.  If
\[
\sup_{\phi\in\Phi_n}|\widehat R_{\mathrm{EIF}}(\phi)-R_{\mathrm{EIF}}(\phi)|
=O_{\Pp}\{(\log M_n/n_{\mathrm{val}})^{1/2}\},
\]
and $\widehat\phi$ minimizes $\widehat R_{\mathrm{EIF}}$ up to optimization error $\epsilon_n$, then
\[
R_{\mathrm{EIF}}(\widehat\phi)
\le \inf_{\phi\in\Phi_n}R_{\mathrm{EIF}}(\phi)
+O_{\Pp}\{(\log M_n/n_{\mathrm{val}})^{1/2}\}+\epsilon_n.
\]
\end{drfdlproposition}

Thus the selected candidate is competitive for the validation EIF criterion.  It need not have small representation error unless the library contains a representation that is already adequate for the nuisance functions.

\subsection{Targeted fluctuation}

The optional targeted update is performed inside each outer training fold and is applied backward in time.  It follows the targeting logic of TMLE introduced by \cite{vanderlaan2006targeted}, adapted here to the sequential pseudo-outcome recursion implemented in longitudinal targeted learning \citep{lendle2017ltmle}.  At time $t$, assume that a pseudo-target $\widehat Y_{i,t+1}^{g}$ has already been constructed on the inner training data.  Let
\[
\mathcal I_t=\{i:\bar R_{i,t-1}=1\},\qquad
\widehat C_{it}^{g}=\frac{\ind\{\bar A_{i,t}=\bar g_t(\bar H_{i,t}),\bar R_{i,t}=1\}}
{\prod_{s=0}^t\widehat\pi_s^g(H_{is})\widehat\rho_s^g(H_{is})},
\]
with the convention that $\bar R_{i,-1}=1$.  For a bounded or binary pseudo-target, the fluctuation parameter is the intercept coefficient in the weighted offset regression
\[
\widehat\epsilon_t=\operatorname*{argmin}_{\epsilon}
\sum_{i\in\mathcal I_t}
\widehat C_{it}^{g}
\,\ell_{\mathrm{log}}
\left(\widehat Y_{i,t+1}^{g},
\logit\widehat Q_t(H_{it},g_t(H_{it}))+\epsilon\right),
\]
where $\ell_{\mathrm{log}}$ is the Bernoulli negative log-likelihood after the pseudo-target is scaled to $[0,1]$.  The updated target-regime prediction is
\begin{equation}
\logit\widehat Q_t^{\widehat\epsilon}(H_t,g_t(H_t))
=
\logit\widehat Q_t(H_t,g_t(H_t))+\widehat\epsilon_t.
\label{eq:fluctuation}
\end{equation}
For continuous outcomes, use the additive least-squares fluctuation
\[
\widehat\epsilon_t=\frac{
\sum_{i\in\mathcal I_t}\widehat C_{it}^{g}
\{\widehat Y_{i,t+1}^{g}-\widehat Q_t(H_{it},g_t(H_{it}))\}}
{\sum_{i\in\mathcal I_t}\widehat C_{it}^{g}+10^{-8}},\qquad
\widehat Q_t^{\widehat\epsilon}=\widehat Q_t+\widehat\epsilon_t.
\]
Equivalently, one may fit the fluctuation only on the regime-compatible subset with inverse-probability weights; the weighted formulation above writes the same restriction as a full-risk-set regression.  The pseudo-target, offset, weights, and selected fluctuation are all computed inside the outer training sample.  The outer held-out fold is used only for final pseudo-outcome evaluation.

\subsection{Overlap and robust aggregation}\label{subsec:robust}

If strict positivity is plausible, clipping is used only as finite-sample stabilization with a threshold $\tau_n\to\infty$.  Weight truncation is common in marginal structural modeling, but it changes the finite-sample estimating equation and must be reported transparently, as emphasized by \cite{cole2008constructing}.  For a fixed threshold, replacing $\mathcal H_t^g$ by $\min(\mathcal H_t^g,\tau)$ changes the estimating equation and is therefore a sensitivity analysis rather than, by default, the efficient influence function of a causal target.  If empirical overlap is intrinsically poor, an explicit overlap estimand is preferable to post hoc truncation of the full-population estimating equation; positivity diagnostics are therefore part of the estimand-definition problem, not only a numerical detail \citep{petersen2012positivity}.  One transparent option is the prespecified overlap-weighted target
\begin{equation}
\theta_{\omega}(g)=\frac{\E\{\omega(H_0)Y^g\}}{\E\{\omega(H_0)\}},
\label{eq:overlap-target}
\end{equation}
where $\omega$ is chosen before looking at outcomes and downweights histories with little support.  Real-data analyses therefore report estimates across clipping thresholds, but state clearly whether inference targets $\theta(g)$, a regularized clipped estimating equation, or the causal overlap target $\theta_\omega(g)$.

Let $\widehat Z_i^g=\widetilde Y_{i0}^g(\widehat\eta^{(-k(i))})$ denote the out-of-fold pseudo-outcome.  The mean estimator and its estimated influence score are
\[
\widehat\theta_M(g)=\Pn\widehat Z_i^g,\qquad
\widehat\varphi_i^g=\widehat Z_i^g-\widehat\theta_M(g).
\]
For a contrast $\Delta(g_1,g_0)=\theta(g_1)-\theta(g_0)$, use
\begin{equation}
\widehat\Delta_M=\widehat\theta_M(g_1)-\widehat\theta_M(g_0),\qquad
\widehat\sigma_\Delta^2=\Pn(\widehat\varphi_i^{g_1}-\widehat\varphi_i^{g_0})^2 .
\label{eq:contrast-variance}
\end{equation}
The Wald interval is based on \eqref{eq:contrast-variance} only for mean aggregation.

For heavy-tailed outcomes or unstable pseudo-outcomes, we also report a Catoni estimator.  This plays the same bounded-influence role for the final pseudo-outcome mean that \cite{huber1964robust} introduced for robust location estimation, while \cite{catoni2012challenging} supplies nonasymptotic mean-deviation motivation.  With scale $\alpha_g>0$, the Catoni estimate solves
\[
\sum_{i=1}^n \psi\{\alpha_g(\widehat Z_i^g-\theta)\}=0,\qquad
\psi(x)=
\begin{cases}
\log(1+x+x^2/2), & x\ge 0,\\
-\log(1-x+x^2/2), & x<0.
\end{cases}
\]
A median-of-means aggregate can be used similarly.  These robust aggregates act after orthogonalization; they improve point stability but do not inherit the Wald variance in \eqref{eq:contrast-variance}.

\section{Algorithm}\label{sec:algorithm}

Algorithm~\ref{alg:drfdl} summarizes the implementation in the same cross-fitted style as the estimator.  All representation selection, calibration, early stopping, and optional targeting are performed inside the outer training sample.  Held-out folds are used only for pseudo-outcome evaluation, following the sample-splitting logic used by \cite{chernozhukov2018dml} for modern orthogonal-score estimation.  Robust aggregation is reported as a sensitivity analysis unless additional robust-inference conditions are verified.

\begin{algorithm}[t]
\caption{Cross-fitted DR-FRL for one target regime}\label{alg:drfdl}
\begin{algorithmic}[1]
\STATE \textbf{Input:} data $\{O_i\}_{i=1}^n$, regime $g$, folds $K$, clipping grid $\mathcal C_{\rm clip}$, aggregation rule $\mathcal A$.
\STATE Randomly partition $\{1,\ldots,n\}$ into folds $I_1,\ldots,I_K$.
\FOR{$k=1,\ldots,K$}
    \STATE Split $D_{-k}=\{O_i:i\notin I_k\}$ into inner fitting and validation samples.
    \STATE Train $e_\phi^X,e_\phi^H$ and nuisance heads $(\widehat Q_t,\widehat\pi_t,\widehat\rho_t)_{t=0}^T$ on the inner fitting sample.
    \STATE Select the representation by the validation criterion in \eqref{eq:targeted-loss}; calibrate $\widehat\pi_t,\widehat\rho_t$ and optionally apply \eqref{eq:fluctuation}.
    \FOR{each $i\in I_k$}
        \STATE Set $\widetilde Y_{i,T+1}^{g}\leftarrow Y_i$.
        \FOR{$t=T,T-1,\ldots,0$}
            \STATE Set $a_{it}^{g}\leftarrow g_t(H_{it})$ and compute the clipped or stabilized increment $\widehat\omega_{it,\tau}^{g}$.
            \STATE Update $\widetilde Y_{it}^{g}\leftarrow \widehat Q_t^{(-k),g}(H_{it})+\widehat\omega_{it,\tau}^{g}\{\widetilde Y_{i,t+1}^{g}-\widehat Q_t^{(-k),g}(H_{it})\}$.
        \ENDFOR
        \STATE Store the out-of-fold pseudo-outcome $\widehat Z_i^g\leftarrow \widetilde Y_{i0}^{g}$.
    \ENDFOR
\ENDFOR
\STATE For mean aggregation, set $\widehat\theta_g\leftarrow n^{-1}\sum_i\widehat Z_i^g$ and $\widehat\sigma_g^2\leftarrow n^{-1}\sum_i(\widehat Z_i^g-\widehat\theta_g)^2$.
\STATE Report the Wald interval $\widehat\theta_g\pm z_{1-\alpha/2}\widehat\sigma_g/\sqrt n$ only for mean aggregation.
\IF{$\mathcal A$ is Catoni or median-of-means}
    \STATE Report the robust aggregate as a point-estimate sensitivity analysis; intervals are exploratory unless separately justified.
\ENDIF
\STATE \textbf{Output:} $\widehat\theta_g$, interval or sensitivity summary, clipping sensitivity, calibration, effective sample sizes, and pseudo-outcome diagnostics.
\end{algorithmic}
\end{algorithm}

\paragraph{Computational complexity.}
If a Set Transformer with $r$ inducing points and embedding dimension $d_x$ is used at each visit, the per-visit cost is $O(M_t r d_x+r^2d_x)$.  A full self-attention encoder over $T+1$ visits has cost $O(T^2d_s)$ per subject.  The total training cost is
\[
O\left(K E n\left[\sum_{t=0}^T C_X(M_t,d_x,r)+T^2d_s+C_{\mathrm{heads}}\right]\right),
\]
where $E$ is the number of epochs.  The replication archive records random seeds, model-library definitions, calibration choices, and hardware-relevant configuration details; the simulation runner also prints elapsed time for local runtime auditing.

\section{Theory}\label{sec:theory}

All nuisance estimators are cross-fitted, so evaluation observations are independent of the nuisance training sample conditional on the fold split.  The results below include treatment and censoring.  The theory follows the semiparametric expansion logic of efficient influence functions and orthogonal estimating equations developed by \cite{newey1994semiparametric} and \cite{tsiatis2006semiparametric}.  The main text states the assumptions and conclusions needed to interpret the estimator; the supplementary material gives the complete algebraic proofs, including the telescoping identities, projection-stability argument, and product-remainder derivation.

\begin{assumption}[Moment and boundedness]\label{ass:moments}
The EIF $\phi_g^\star$ has finite variance.  Outcome regressions are square integrable.  Under strict-positivity inference, true treatment and censoring probabilities are bounded away from zero.  Under clipped inference, clipped weights satisfy the moment conditions in Theorem~\ref{thm:clip}.
\end{assumption}

\begin{assumption}[Cross-fitted nuisance convergence]\label{ass:rates}
For the representation selected in a training fold, let
$Q_{t,\widehat\phi}^{\dagger}$ and $\Gamma_{t,\widehat\phi}^{\dagger}$ denote the best approximations to $Q_t^{g,\star}$ and $\Gamma_t^{g,\star}=(\pi_t^{g,\star},\rho_t^{g,\star})$ within the corresponding representation-indexed nuisance classes.  Define stochastic estimation errors
\[
r_{Q,t}=\norm{\widehat Q_t^g-Q_{t,\widehat\phi}^{\dagger}}_{L_2(P_0)},\qquad
r_{\Gamma,t}=\norm{\widehat\Gamma_t^g-\Gamma_{t,\widehat\phi}^{\dagger}}_{L_2(P_0)}.
\]
The estimators are cross-fitted and satisfy $\max_t(r_{Q,t}+r_{\Gamma,t})=o_{\Pp}(1)$.
\end{assumption}

\begin{assumption}[EIF-map stability]\label{ass:eifstability}
Let $D_g(O;\eta,\theta)$ denote the EIF map defined by \eqref{eq:telescoping}.  Conditional on the outer training folds,
\[
\norm{D_g(\cdot;\widehat\eta,\theta(g))-D_g(\cdot;\eta^\star,\theta(g))}_{L_2(P_0)}=o_{\Pp}(1).
\]
A primitive sufficient condition is that estimated treatment and censoring probabilities are uniformly bounded away from zero after calibration or deterministic clipping, the total nuisance errors $\norm{\widehat Q_t-Q_t^\star}_2$ and $\norm{\widehat\Gamma_t-\Gamma_t^\star}_2$ converge to zero for each fixed $t$, and the residuals in the EIF map have uniformly bounded second moments.  This condition is used only to control the cross-fitted empirical process term.
\end{assumption}

\begin{assumption}[Representation approximation and product rate]\label{ass:representation}
For the learned representation,
\[
\sum_{t=0}^T \{r_{Q,t}+b_{Q,t}(\widehat\phi)\}\{r_{\Gamma,t}+b_{\Gamma,t}(\widehat\phi)\}=o_{\Pp}(n^{-1/2}).
\]
If $S_t$ is exactly causally sufficient, all $b_{Q,t}$ and $b_{\Gamma,t}$ are zero.  This is a high-level product-rate assumption on the representation selected inside the outer training sample.  Proposition~\ref{prop:selector} explains what the EIF-targeted validation objective can guarantee for a finite candidate library, but the product-rate condition itself remains the condition needed for first-order inference.
\end{assumption}

For later reference, write
\[
a_{Q,t}=r_{Q,t}+b_{Q,t}(\widehat\phi),\qquad
a_{\Gamma,t}=r_{\Gamma,t}+b_{\Gamma,t}(\widehat\phi).
\]
The condition is $\sum_t a_{Q,t}a_{\Gamma,t}=o_{\Pp}(n^{-1/2})$.  It allows one nuisance side to be slower when the other side is accurate, but it does not allow both outcome-side and mechanism-side representation errors to be large.

\begin{assumption}[Targeting and calibration stability]\label{ass:calibration}
Representation selection, calibration, and targeted fluctuation are performed inside the outer training sample and are independent of the final evaluation folds.  If an inner validation criterion is optimized repeatedly, a separate inner selection or calibration split is used so that the selected nuisance vector remains measurable with respect to the outer training sample.  Fluctuation parameters remain in a shrinking neighborhood of zero when initial nuisance estimators are consistent, and calibrated probabilities remain bounded away from zero after clipping.
\end{assumption}

\begin{drfdllemma}[Projection decomposition and EIF-map stability]\label{lem:projection-stability}
For the representation selected in an outer training fold,
\[
\norm{\widehat Q_t-Q_t^{g,\star}}_2
\le r_{Q,t}+b_{Q,t}(\widehat\phi),\qquad
\norm{\widehat\Gamma_t-\Gamma_t^{g,\star}}_2
\le r_{\Gamma,t}+b_{\Gamma,t}(\widehat\phi).
\]
If the treatment and censoring probabilities are bounded away from zero after calibration or deterministic clipping and the residuals in the EIF map have bounded second moments, these total-error bounds imply Assumption~\ref{ass:eifstability} whenever $\max_t\{r_{Q,t}+b_{Q,t}+r_{\Gamma,t}+b_{\Gamma,t}\}=o_{\Pp}(1)$ for fixed $T$.
\end{drfdllemma}

\begin{drfdltheorem}[Observed-history identification and EIF]\label{thm:identification}
Under Assumptions~\ref{ass:consistency}-\ref{ass:observed-history} and positivity, the observed-history value $\theta(g)$ is identified by \eqref{eq:gformula}.  In the nonparametric observed-data model for the observed coarsened histories, the efficient influence function is \eqref{eq:eif}, with residuals conditioned on $(H_t,A_t=g_t(H_t),R_t=1,\bar R_{t-1}=1)$.
\end{drfdltheorem}

This is an observed-history result.  Assumption~\ref{ass:exchangeability} identifies the future pre-treatment history used by later regime decisions, not only the terminal outcome.  A latent full-functional-history target would require an additional measurement or coarsening model; its canonical gradient need not equal \eqref{eq:eif}.

\begin{drfdltheorem}[Recursive sequential multiple robustness]\label{thm:sdr}
Let $\widetilde Y_0^g(\eta)$ be the recursion in \eqref{eq:pseudo-recursion}.  Suppose the identification assumptions hold and working nuisance limits are bounded.  Define $\Delta_{T+1}=0$ and, for $t=T,\ldots,0$, let $\Delta_t(H_t)=E\{\widetilde Y_t^g(\eta)\mid H_t\}-Q_t^{g,\star}(H_t)$.  If the following condition holds backward recursively, then $E\{\widetilde Y_0^g(\eta)\}=\theta(g)$: whenever $\Delta_{t+1}=0$, either
\[
Q_t^g=Q_t^{g,\star}\qquad\text{or}\qquad (\pi_t^g,\rho_t^g)=(\pi_t^{g,\star},\rho_t^{g,\star})
\]
almost surely on the regime support.  Consequently, the cross-fitted DR-FRL estimator is consistent whenever the probability limits of the nuisance learners satisfy this recursive condition.
\end{drfdltheorem}

The common shorthand ``at each time, either the outcome nuisance or the mechanism nuisance is correct'' should be read through this backward recursion.  For $T=1$, the terminal step must first be unbiased; only then does the time-0 step require one of its two nuisance components to be correct.

\begin{drfdltheorem}[Second-order remainder with representation error]\label{thm:remainder}
Under Assumptions~\ref{ass:consistency}-\ref{ass:calibration}, including the EIF-map stability condition in Assumption~\ref{ass:eifstability}, the DR-FRL-Mean estimator satisfies
\begin{equation}
\widehat\theta_g-\theta(g)=(\Pn-P_0)\phi_g^\star+R_{2,n}+R_{\mathrm{rep},n}+R_{\mathrm{clip},n}+o_{\Pp}(n^{-1/2}),
\label{eq:expansion}
\end{equation}
where $R_{\mathrm{clip},n}=0$ for unclipped strict-positivity inference and
\begin{equation}
|R_{2,n}+R_{\mathrm{rep},n}|
\le C\sum_{t=0}^T
\{r_{Q,t}+b_{Q,t}(\widehat\phi)\}
\{r_{\Gamma,t}+b_{\Gamma,t}(\widehat\phi)\}
+o_{\Pp}(n^{-1/2}).
\label{eq:product-remainder}
\end{equation}
\end{drfdltheorem}

Representation error is therefore controlled through the same orthogonal product structure as nuisance error; it is not automatically harmless.  Lemma~\ref{lem:projection-stability} supplies the projection step, and the condition matches the broader message of \cite{farrell2021deep}: flexible learners support inference only when target-relevant approximation and stochastic errors are controlled.  Equivalently,
\[
a_{Q,t}a_{\Gamma,t}
=r_{Q,t}r_{\Gamma,t}
+r_{Q,t}b_{\Gamma,t}
+b_{Q,t}r_{\Gamma,t}
+b_{Q,t}b_{\Gamma,t}.
\]
The last three terms are the price of replacing $H_t$ by $S_t$.

The next theorem covers only the mean-aggregation estimator.  Robust Catoni and median-of-means aggregation are deliberately separated in Theorem~\ref{thm:robust}, because bounded-influence concentration for feasible pseudo-outcomes is not the same statement as asymptotic normality with plug-in variance estimation.

\begin{drfdltheorem}[Asymptotic linearity and Wald inference]\label{thm:an}
If Assumptions~\ref{ass:moments}-\ref{ass:calibration} hold and the product condition in Assumption~\ref{ass:representation} is satisfied, then
\[
\sqrt n\{\widehat\theta_g-\theta(g)\}
=\frac{1}{\sqrt n}\sum_{i=1}^n\phi_g^\star(O_i)+o_{\Pp}(1)
\rightsquigarrow N(0,\sigma_g^2),
\]
where $\sigma_g^2=\E_0[(\phi_g^\star)^2]$.  The empirical variance of the out-of-fold estimated influence scores is consistent for $\sigma_g^2$ when mean aggregation is used.
\end{drfdltheorem}

\begin{drfdltheorem}[Clipping and overlap targets]\label{thm:clip}
If strict positivity holds and the clipping level $\tau_n\to\infty$ with $P(|\mathcal H_t^{g,\star}|>\tau_n)\to0$ sufficiently fast, clipping is asymptotically negligible for $\theta(g)$.  If the target is instead the overlap value $\theta_\omega(g)$ in \eqref{eq:overlap-target}, its EIF is obtained by the quotient rule applied to the weighted numerator and denominator.
\end{drfdltheorem}

Thus clipping is a regularization device unless an overlap estimand is defined in advance.  Replacing $\mathcal H_t^g$ by $\min(\mathcal H_t^g,\tau)$ is not, by itself, the canonical gradient of a new causal target.

\begin{drfdltheorem}[Robust aggregation as concentration, not Wald inference]\label{thm:robust}
Let $Z_i^g=\widetilde Y_{i0}^g(\eta^\star)$ be the oracle pseudo-outcome with $E Z_i^g=\theta(g)$ and $\Var(Z_i^g)\le v<\infty$.  Conditional on the cross-fitting splits, let $\widehat Z_i^g$ be the feasible out-of-fold pseudo-outcome and suppose $n^{-1}\sum_i|\widehat Z_i^g-Z_i^g|=\delta_n=o_{\Pp}(1)$.  With a Catoni scale of order $\sqrt{\log(1/\delta)/(nv)}$, the feasible Catoni aggregate obeys
\[
|\widehat\theta_C-\theta(g)|
=O_{\Pp}\{\sqrt{v\log(1/\delta)/n}\}+O_{\Pp}(\delta_n).
\]
If only a finite $(1+\kappa)$ moment is assumed, a median-of-means aggregate with blocks formed independently of the evaluation observations is consistent at the usual polynomial rate, up to the same feasible-pseudo-outcome discrepancy.
\end{drfdltheorem}

The condition on $\delta_n$ is an additional stability requirement for robust aggregation.  The result gives concentration for robust point estimation; it does not claim asymptotic normality, plug-in variance consistency, or bootstrap validity for Catoni or median-of-means aggregation.  Wald inference in Theorem~\ref{thm:an} is stated only for mean aggregation.

\begin{drfdlcorollary}[A sufficient rate condition for basis encoders]\label{cor:sieve}
Suppose the latent functional covariate has a Karhunen-Loeve expansion with eigenvalues decaying as $\lambda_k\lesssim k^{-2a}$, nuisance functions are smooth functions of the first $K_n$ scores and scalar histories, and scores are estimated from sparse noisy samples with mean integrated squared error $\epsilon_{n,K}^2$.  A spline or FPCA-sieve encoder satisfies Assumption~\ref{ass:representation} whenever
\[
\sum_t \{r_{Q,t}+K_n^{-s}+\epsilon_{n,K}\}\{r_{\Gamma,t}+K_n^{-s}+\epsilon_{n,K}\}=o_{\Pp}(n^{-1/2}).
\]
\end{drfdlcorollary}

This corollary is motivated by sparse functional principal component theory \citep{yao2005functional} and functional regression \citep{morris2015functional}.  It gives one concrete sieve path for verifying the high-level representation condition, not a primitive rate theorem for every neural encoder.

\paragraph{Two extensions.}
Fixed-state efficiency and growing-horizon inference are useful benchmarks but are not part of the main theorem sequence.  If a sufficient state is fixed a priori or learned on an auxiliary sample and the target regime is measurable with respect to that state, standard one-step or TMLE arguments can yield the observed-history efficiency bound under the same product-rate conditions \citep{vanderlaan2012targeted}.  This is a fixed-state benchmark, not a claim that an arbitrary data-adaptive neural representation is automatically efficient.  Similarly, growing-horizon asymptotic linearity can be obtained by replacing the fixed-$T$ product-rate condition with a summable condition over $t\le T_n$ and by imposing a triangular-array central limit theorem for the EIF scores.  These two cases are discussed as future extensions rather than central results.

\section{Simulation Study}\label{sec:simulation}

The simulation study is a set of stress tests rather than a ranking contest over every possible longitudinal estimator.  The data-generating mechanisms create settings in which an estimand-targeted representation should matter: nonlinear feedback, irregular functional measurement, heavy-tailed outcomes, weak overlap, informative sampling, and high-dimensional functional structure.  The scenarios ask four questions:
\begin{itemize}[leftmargin=1.8em]
\item Does an observed-history functional state help when confounding is carried by low-variance or high-dimensional functional features?
\item Does the EIF-based recursion stabilize estimation relative to pure weighting or prediction-style baselines?
\item Does bounded-influence final aggregation reduce extreme pseudo-outcome errors under heavy tails and weak support?
\item Do overlap and coverage diagnostics separate poor point estimation from conservative intervals caused by unstable weights?
\end{itemize}
The goal is to test the workflow under controlled failures of simple adjustment.  It is not a full factorial ablation over every encoder, validation loss, and aggregation rule.

\subsection{Design and target regimes}

The main simulation uses $n=400$ subjects, $T+1=6$ decision times, 150 Monte Carlo replications, and seed 20260503.  Scenario G uses $K^\star=12$ latent scores; all other scenarios use $K^\star=6$.  Population truths are computed from independent intervention samples of size $60{,}000$ under each policy.

For subject $i$ and time $t=0,\ldots,5$, the latent functional covariate is
\[
X_{it}(u)=\sum_{k=1}^{K^\star}\alpha_{itk}\psi_k(u),\qquad u\in[0,1],
\]
where $\psi_k$ are Fourier basis functions.  Let $\sigma_k=0.80^{k-1}$.  The score process satisfies
\[
\alpha_{i0k}\sim N(0,\sigma_k^2),\qquad
\alpha_{itk}=0.55\alpha_{i,t-1,k}+\{1-0.55^2\}^{1/2}\varepsilon_{itk},\quad
\varepsilon_{itk}\sim N(0,\sigma_k^2).
\]
Baseline covariates are $B_{i1}\sim N(0,1)$, $B_{i2}\sim\mathrm{Bernoulli}(0.48)$, and $B_{i3}\sim N(0,1)$.  We set $A_{i,-1}=0$.  The baseline functional signal entering the scalar history is
\[
f_{it}=0.60\alpha_{it1}-0.35\alpha_{it2}+0.25\sin(\alpha_{it3}),
\]
with scenario-specific additions stated below.  The scalar pre-treatment history is generated as
\begin{align*}
L_{it0}&=f_{it}+0.25B_{i1}+0.15A_{i,t-1}+\epsilon_{it0}, & \epsilon_{it0}&\sim N(0,0.35^2),\\
L_{it1}&=0.65\tanh(\alpha_{it2})+0.25B_{i2}+0.12t+\epsilon_{it1}, & \epsilon_{it1}&\sim N(0,0.25^2),\\
L_{it2}&=\cos(\alpha_{it3})+0.10A_{i,t-1}+\epsilon_{it2}, & \epsilon_{it2}&\sim N(0,0.20^2),\\
L_{it3}&=t/5 .
\end{align*}
The irregular measurement count is $m_{it}=2+\mathrm{Poisson}\{\max(\lambda_{it},0.5)\}$, where $\lambda_{it}=12\,\mathrm{expit}(0.10+0.25L_{it0})+2$ except in Scenario F, where $\lambda_{it}=12\,\mathrm{expit}(0.25+0.65L_{it0})+2$.  Observed point-cloud summaries are the noisy mean, minimum, maximum, standard deviation, count, normalized count, and severity-weighted count derived from these measurements.  For example,
\[
\bar X_{it}=\alpha_{it1}+0.10\alpha_{it2}+\zeta_{it},\qquad
\zeta_{it}\sim N(0,0.15^2/m_{it}).
\]

Treatment follows
\[
\logit P(A_{it}=1\mid H_{it})=-0.25+s\{0.45L_{it0}+0.30\tanh(L_{it1})+0.35A_{i,t-1}+0.20B_{i1}+0.18\bar X_{it}\},
\]
where $s=1$ except in Scenario D, where $s=2.25$.  Scenarios B and G add $s\{0.30L_{it0}\bar X_{it}+0.20\sin(\alpha_{it3})\}$, and Scenario F adds $0.35(m_{it}-\bar m_t)/\mathrm{sd}(m_t)$.  The outcome is
\[
Y_i=2+0.25B_{i1}-0.10B_{i2}
+\sum_{t=0}^5\{0.32L_{it0}-0.20L_{it1}+0.10\bar X_{it}+\tau_{it}A_{it}-0.04A_{i,t-1}\}+e_i,
\]
where $\tau_{it}=0.65+0.18L_{it0}-0.12A_{i,t-1}$, with an added $0.15\tanh(\alpha_{it3})$ in Scenarios B and G.  Scenario C replaces the normal outcome error by $t_3/\sqrt{3}$ plus 3\% contamination from $N(0,9^2)$.

The target contrast is
\[
\theta(g_{\mathrm{early}})-\theta(g_{\mathrm{delayed}}),
\]
where $g_{\mathrm{early}}$ treats when $L_{it0}>0.05$ or $L_{it1}>0.85$, and $g_{\mathrm{delayed}}$ treats when either $L_{it0}>0.05$ persists from the previous time point or $L_{it1}>1.35$.  The population contrasts are 1.272 (A), 1.267 (B), 1.273 (C), 1.272 (D), 1.254 (F), and 1.253 (G).

\subsection{Scenarios and estimators}

The six main scenarios are:
\begin{itemize}[leftmargin=1.8em]
\item Scenario A: baseline setting with $K^\star=6$, normal errors, treatment strength $s=1$, and severity-dependent but weakly informative measurement intensity.
\item Scenario B: Scenario A plus nonlinear functional interactions in $f_{it}$, treatment, and treatment effect; specifically $0.25\alpha_{it1}\alpha_{it4}+0.20\{\alpha_{it3}^2-E(\alpha_{it3}^2)\}$ enters $f_{it}$.
\item Scenario C: Scenario A with heavy-tailed and contaminated outcome errors.
\item Scenario D: Scenario A with treatment strength $s=2.25$, producing near-positivity stress.
\item Scenario F: Scenario A with strongly severity-dependent measurement intensity and a standardized measurement-count term in the treatment model.
\item Scenario G: Scenario B with $K^\star=12$ and additional low-variance signal $0.35\tanh(\alpha_{it9})-0.25\alpha_{it,11}$ in $f_{it}$.
\end{itemize}

The comparison set includes proposed DR-FRL variants, semiparametric baselines, and sequence-style prediction baselines.  The main tables and heatmaps show a readable subset, while the CSV files and ranking table retain the full method set.  iTMLE is included because it is the robust-aggregation analogue of the low-dimensional targeted-learning baseline in this CPU implementation.  This helps separate final robust aggregation from richer observed-history representation, although it is not a complete component-ablation study.

The longitudinal targeted-learning benchmark follows \cite{lendle2017ltmle}.  The sequence-style baselines are CPU-reproducible proxies for recurrent and attention-based counterfactual prediction approaches such as \cite{melnychuk2022causaltransformer}; they are not full GPU reproductions of the original architectures.  The latent-score SDR diagnostic receives the true latent functional scores, but still estimates finite-sample nuisance functions and therefore is not an oracle upper bound.  Coverage entries for Catoni aggregation and sequence-style baselines are calibration diagnostics; the Wald theorem applies to DR-FRL-Mean.

\subsection{Numerical results}

Tables~\ref{tab:sim-abc} and~\ref{tab:sim-dfg} give the main numerical results.  Each entry reports bias, RMSE, and diagnostic coverage.  DR-FRL-Mean is the version covered by the Wald theorem; DR-FRL-Catoni uses the same pseudo-outcomes but replaces the final mean by bounded-influence aggregation.  The comparison separates two effects.  Functional states matter most in Scenarios F and G, where measurement intensity or low-variance high-dimensional scores carry confounding.  Catoni aggregation mainly reduces point error under heavy tails and weak support.  IPTW is least stable because the dynamic policy creates sparse treatment-path matches and unstable cumulative weights; high nominal coverage in those rows reflects wide intervals, not good point estimation.

\begin{table}[H]
\centering
\caption{Results for Scenarios A--C ($n=400$). Entries are bias/RMSE/diagnostic coverage.}
\label{tab:sim-abc}
\scalebox{0.8}{
\begin{tabular}{lccc}
\toprule
Method & A: base & B: nonlinear & C: heavy tail \\
\midrule
DR-FRL-Mean & -0.19/0.24/0.82 & -0.19/0.24/0.85 & -0.20/0.31/0.93 \\
DR-FRL-Catoni & -0.14/0.19/0.91 & -0.13/0.19/0.91 & -0.16/0.26/0.97 \\
LTMLE & -0.22/0.25/0.88 & -0.20/0.24/0.94 & -0.20/0.30/0.93 \\
iTMLE & -0.19/0.22/0.93 & -0.19/0.22/0.97 & -0.17/0.24/0.97 \\
G-Net-style & -0.25/0.26/0.03 & -0.26/0.27/0.01 & -0.23/0.29/0.09 \\
Transformer-style & -0.21/0.23/0.05 & -0.20/0.22/0.07 & -0.20/0.28/0.14 \\
RMSN-style & -0.24/0.26/0.03 & -0.24/0.26/0.03 & -0.22/0.28/0.12 \\
Latent-score SDR & -0.22/0.26/0.79 & -0.23/0.27/0.73 & -0.20/0.30/0.94 \\
IPTW & 1.01/1.46/0.98 & 1.18/1.61/0.95 & 1.03/1.57/0.97 \\
G-computation & -0.23/0.25/0.03 & -0.24/0.25/0.03 & -0.22/0.28/0.12 \\
\bottomrule
\end{tabular}}
\end{table}

\begin{table}[H]
\centering
\caption{Results for Scenarios D--G ($n=400$). Entries are bias/RMSE/diagnostic coverage.}
\label{tab:sim-dfg}
\scalebox{0.8}{
\begin{tabular}{lccc}
\toprule
Method & D: near positivity & F: informative sampling & G: high-dimensional functional \\
\midrule
DR-FRL-Mean & -0.19/0.24/0.95 & -0.17/0.22/0.91 & -0.20/0.24/0.87 \\
DR-FRL-Catoni & -0.20/0.24/0.96 & -0.12/0.17/0.96 & -0.14/0.19/0.95 \\
LTMLE & -0.21/0.25/0.98 & -0.17/0.22/0.97 & -0.21/0.25/0.94 \\
iTMLE & -0.31/0.33/0.99 & -0.16/0.20/0.99 & -0.20/0.23/0.97 \\
G-Net-style & -0.28/0.30/0.01 & -0.21/0.23/0.05 & -0.26/0.28/0.01 \\
Transformer-style & -0.22/0.24/0.04 & -0.18/0.21/0.13 & -0.20/0.23/0.07 \\
RMSN-style & -0.28/0.30/0.01 & -0.20/0.22/0.05 & -0.25/0.27/0.01 \\
Latent-score SDR & -0.25/0.28/0.85 & -0.19/0.24/0.89 & -0.24/0.26/0.77 \\
IPTW & 1.72/1.88/0.78 & 1.34/1.68/0.93 & 0.98/1.42/0.99 \\
G-computation & -0.28/0.29/0.01 & -0.20/0.22/0.03 & -0.25/0.26/0.01 \\
\bottomrule
\end{tabular}}
\end{table}

Several augmented estimators exhibit negative finite-sample bias.  This pattern is consistent with the design of the stress test rather than with a failure of the EIF identity: the dynamic regimes have limited empirical support in parts of the history space, nuisance learners are regularized, and clipping changes the effective estimating equation when treatment-path probabilities are small.  In addition, the ``latent-score'' row supplies true functional scores only; it does not supply oracle treatment, censoring, or outcome nuisance functions.  The simulations therefore evaluate implementable pipelines under weak support and regularization, while the large-sample theory concerns the observed-history target under the stated product-rate and positivity conditions.

Table~\ref{tab:component-audit} summarizes the implemented component audit.  Values are averages over the six scenarios; lower RMSE, lower absolute bias, and smaller $|\mathrm{coverage}-0.95|$ are better.

\begin{table}[H]
\centering
\caption{Component-audit summary across the six simulation scenarios.}
\label{tab:component-audit}
\scalebox{0.8}{
\setlength{\tabcolsep}{5pt}
\renewcommand{\arraystretch}{1.08}
\begin{tabular}{>{\raggedright\arraybackslash}p{0.44\textwidth}rrrr}
\toprule
Audit row & RMSE & $|\mathrm{bias}|$ & $|\mathrm{cov.}-.95|$ & Cov. \\
\midrule
Low-dimensional SDR (LTMLE/linear SDR) & 0.251 & 0.201 & 0.026 & 0.940 \\
Low-dimensional SDR plus robust aggregation (iTMLE) & 0.239 & 0.202 & 0.026 & 0.968 \\
Latent-score SDR diagnostic & 0.267 & 0.220 & 0.120 & 0.830 \\
DR-FRL with mean aggregation & 0.248 & 0.190 & 0.063 & 0.888 \\
DR-FRL with Catoni aggregation & 0.208 & 0.147 & 0.021 & 0.941 \\
\bottomrule
\end{tabular}}
\end{table}

\FloatBarrier

The table should be read as an audit, not as a new factorial experiment.  Low-dimensional targeted learning is competitive when scalar summaries carry the relevant history.  iTMLE shows that robust final aggregation can help, but it cannot add missing functional information.  DR-FRL-Mean has similar average RMSE to LTMLE but lower average absolute bias.  DR-FRL-Catoni has the best point accuracy, with average RMSE 0.208 and average absolute bias 0.147, because the functional state, orthogonal recursion, and bounded-influence final aggregation work together.

\FloatBarrier

Figure~\ref{fig:rmse-heatmap} gives a compact visual summary of RMSE.  The heatmap emphasizes three features.  First, pure weighting is fragile even with clipping.  Second, the sequence-style baselines are competitive in benign settings but do not dominate under heavy tails or informative measurement.  Third, the Catoni version of DR-FRL is consistently robust because it combines an orthogonal pseudo-outcome with a bounded-influence final aggregation step.

\begin{figure}[H]
\centering
\includegraphics[width=0.85\textwidth]{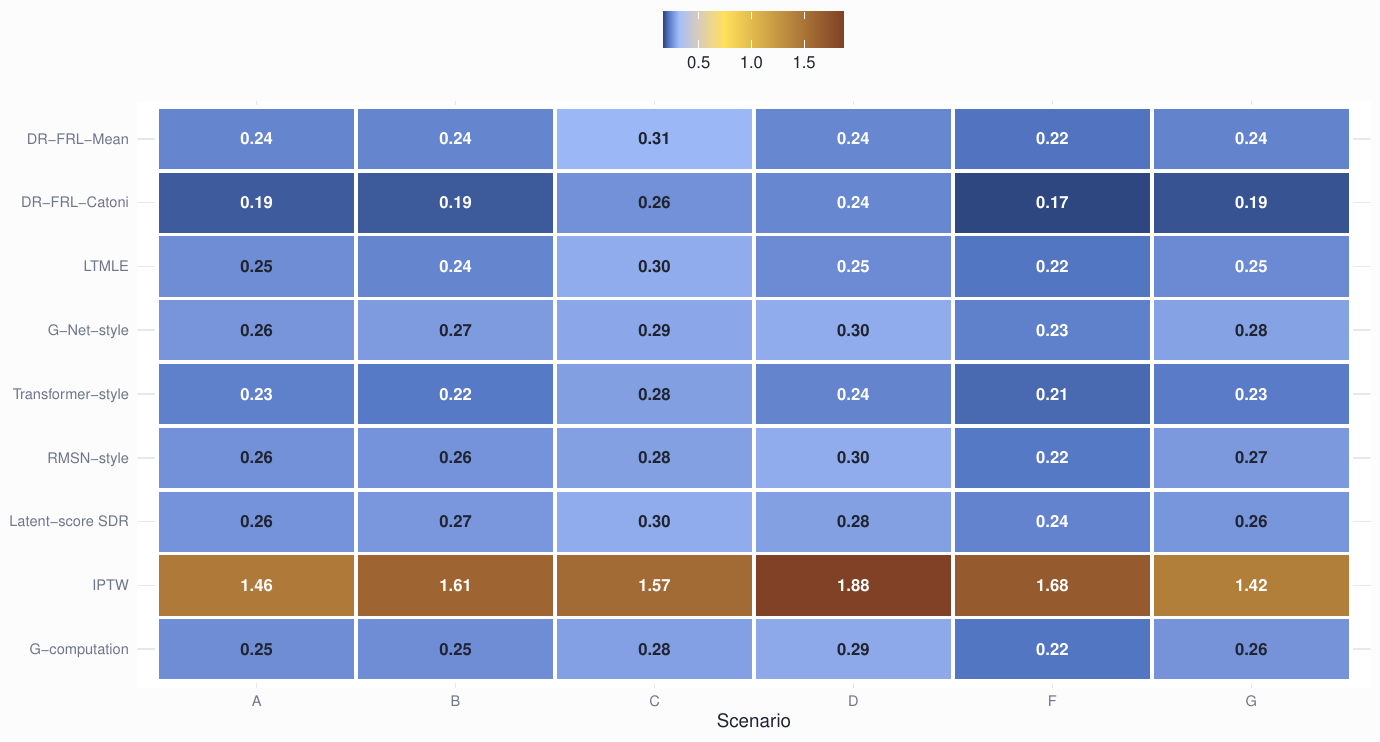}
\caption{RMSE heatmap across the six simulation scenarios.  Lower values are better.}
\label{fig:rmse-heatmap}
\end{figure}

Figure~\ref{fig:coverage-heatmap} reports empirical coverage for the simulation intervals.  The sequence-style baselines are prediction methods rather than semiparametric EIF estimators, so their intervals are best read as calibration diagnostics.  DR-FRL-Catoni gives the strongest joint performance on RMSE and diagnostic coverage among the proposed variants, but Catoni intervals still require additional robust-inference conditions for formal interpretation.  Scenario G is hardest because the nuisance feature space is high-dimensional relative to $n=400$.

\begin{figure}[h]
\centering
\includegraphics[width=0.85\textwidth]{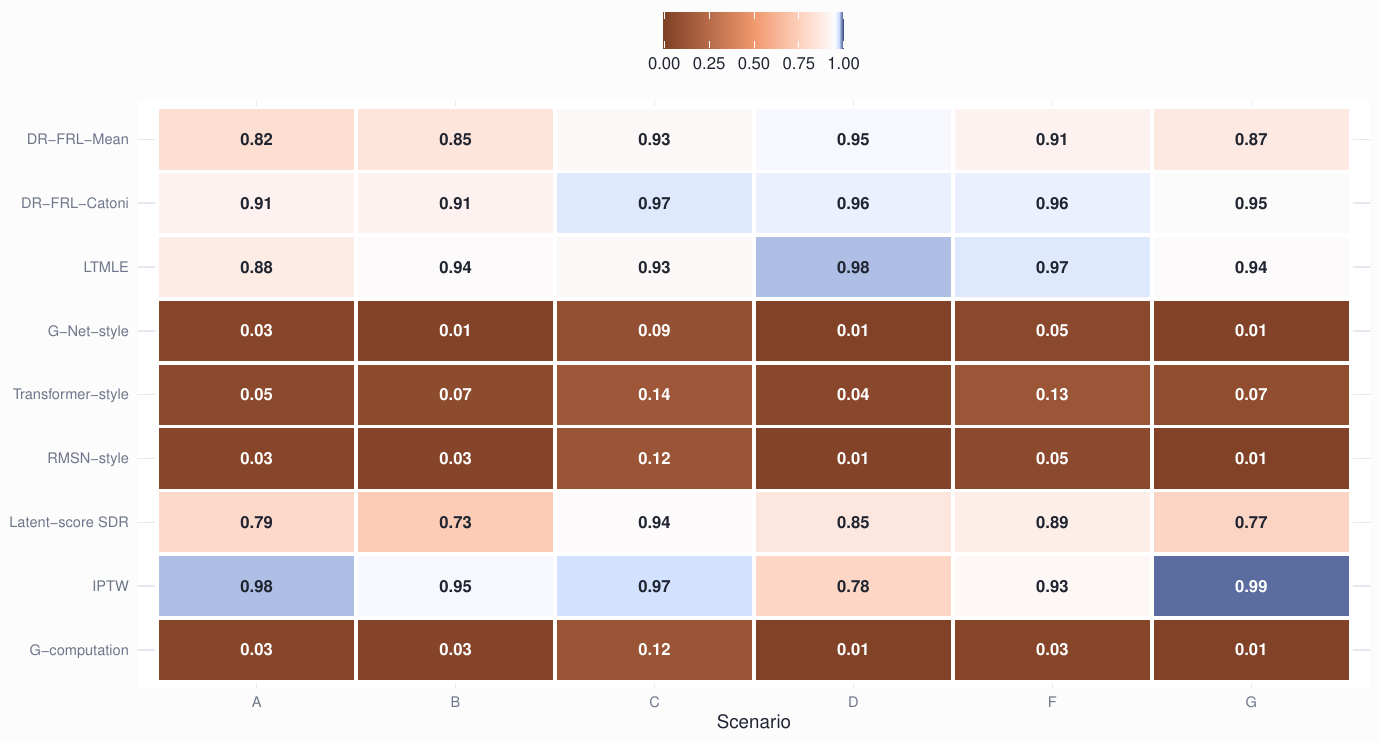}
\caption{Diagnostic empirical coverage across the six simulation scenarios.  The target is 0.95.}
\label{fig:coverage-heatmap}
\end{figure}

\FloatBarrier

\subsection{Robustness and sample size}

Figure~\ref{fig:heavy-tail} focuses on Scenario C. The boxplot shows the full replication-level estimation-error distribution rather than only a summary statistic.  

\begin{figure}[H]
\centering
\includegraphics[width=0.8\textwidth]{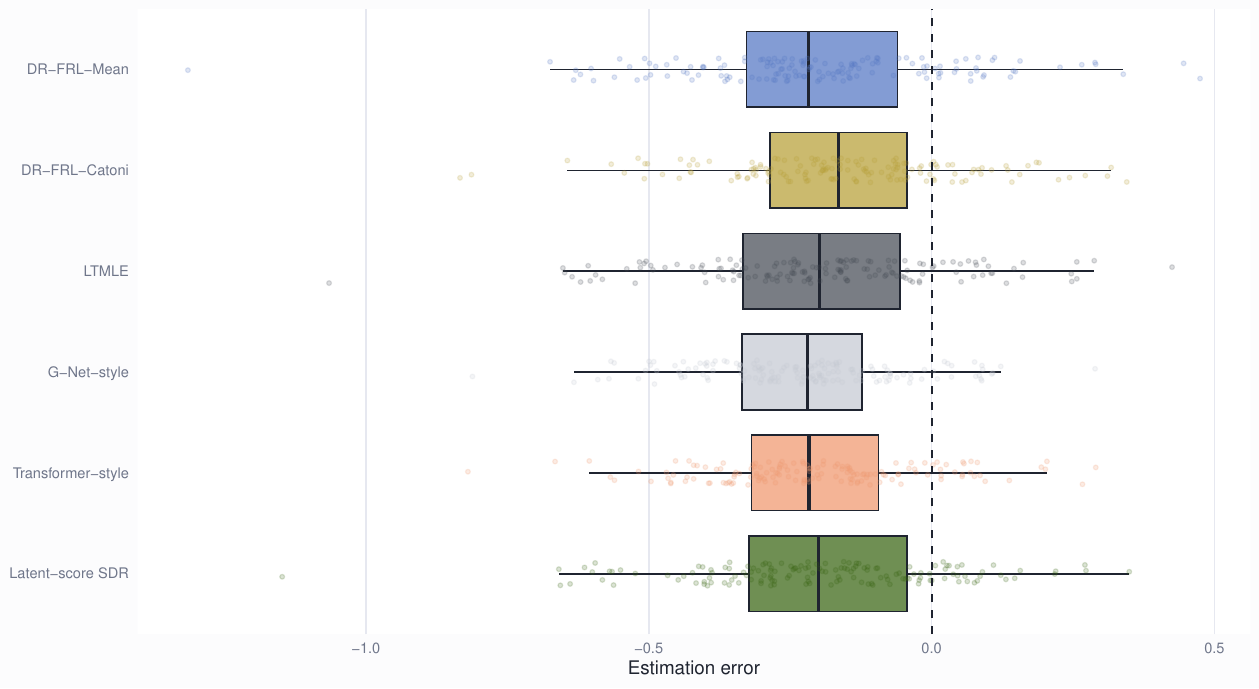}
\caption{Replication estimation errors in Scenario C.}
\label{fig:heavy-tail}
\end{figure}

DR-FRL-Catoni shrinks extreme errors relative to the non-Catoni version, and the interquartile range is smaller than that of LTMLE and the latent-score SDR diagnostic.  This supports the theoretical motivation for robust aggregation: when the pseudo-outcome itself is heavy-tailed, robustifying only the outcome-regression loss is not enough; the final aggregation must also be bounded-influence.

Figure~\ref{fig:ncurve} shows sample-size diagnostics for the base and near-positivity scenarios as a matrix rather than a trend plot.  The top row reports RMSE, and the bottom row reports interval calibration.  Coverage in the Catoni and sequence-style rows remains diagnostic rather than a formal Wald guarantee.  RMSE generally decreases with $n$, but near-positivity remains difficult because increasing sample size helps most when it also improves empirical support for the target regime.

\begin{figure}[h]
\centering
\includegraphics[width=0.9\textwidth]{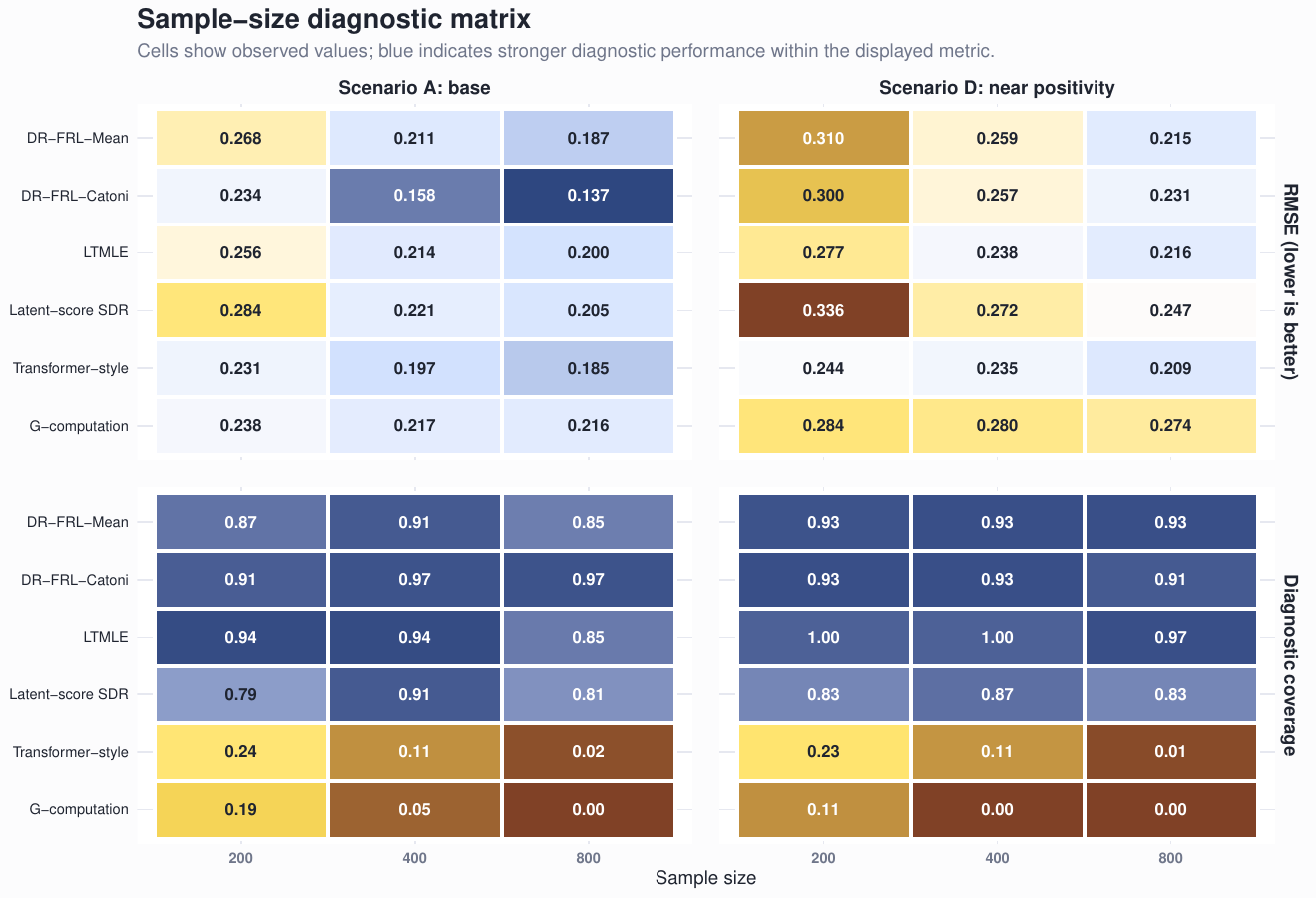}
\caption{Sample-size diagnostics for Scenarios A and D.}
\label{fig:ncurve}
\end{figure}

\begin{table}[H]
\centering
\caption{Sample-size diagnostics. Entries are RMSE/coverage/minimum effective sample size.}
\label{tab:sim-mean-ncurve}
\scalebox{0.9}{
\begin{tabular}{llccc}
\toprule
Scenario & Method & $n=200$ & $n=400$ & $n=800$ \\
\midrule
A: base & DR-FRL-Mean & 0.268/0.867/7.9 & 0.211/0.913/16.1 & 0.187/0.847/34.3 \\
A: base & LTMLE & 0.256/0.940/7.9 & 0.214/0.940/16.2 & 0.200/0.853/34.4 \\
A: base & Latent-score SDR & 0.284/0.793/7.8 & 0.221/0.913/16.1 & 0.205/0.813/34.2 \\
D: near positivity & DR-FRL-Mean & 0.310/0.927/13.9 & 0.259/0.933/27.9 & 0.215/0.933/56.9 \\
D: near positivity & LTMLE & 0.277/1.000/14.3 & 0.238/1.000/28.6 & 0.216/0.967/58.3 \\
D: near positivity & Latent-score SDR & 0.336/0.833/13.1 & 0.272/0.873/26.3 & 0.247/0.833/53.6 \\
\bottomrule
\end{tabular}}
\end{table}
Overall, Table~\ref{tab:sim-mean-ncurve} reinforces the diagnostic message of Figure~\ref{fig:ncurve}: increasing sample size improves point accuracy mainly when it also improves empirical support for the target regime, whereas interval calibration remains sensitive to finite-sample regularization and practical positivity.
\FloatBarrier

\subsection{Diagnostics}

Table~\ref{tab:sim-diagnostics} reports effective sample size, clipping, average standard error, and empirical standard deviation averaged across the six main scenarios.  The diagnostic table confirms that DR-FRL-Mean and LTMLE operate in the same effective-sample-size regime, whereas IPTW has larger pseudo-outcome variance despite high nominal coverage.  Table~\ref{tab:sim-rank} ranks the full simulated method set by RMSE, absolute bias, and closeness of diagnostic coverage to 0.95.  The average rank is not a causal estimand or formal inference score; it is a compact stress-test summary, and its coverage component is exploratory for non-EIF baselines and robust aggregation.  DR-FRL-Catoni ranks first because it performs well simultaneously on point accuracy and interval stability.

\begin{table}[H]
\centering
\caption{Average diagnostics across the six main scenarios. ESS is the mean of the minimum effective sample size across the two dynamic regimes.}
\label{tab:sim-diagnostics}
\scalebox{0.8}{
\begin{tabular}{lrrrr}
\toprule
Method & ESS & Clip\% & ASE & ESD \\
\midrule
DR-FRL-Mean & 18.2 & 6.0 & 0.20 & 0.16 \\
DR-FRL-Catoni & 18.2 & 6.0 & 0.20 & 0.15 \\
LTMLE & 18.3 & 7.0 & 0.24 & 0.15 \\
Latent-score SDR & 17.9 & 5.0 & 0.20 & 0.15 \\
IPTW & 17.8 & 5.7 & 1.65 & 1.03 \\
\bottomrule
\end{tabular}}
\end{table}

\begin{table}[H]
\centering
\caption{Average diagnostic rank across scenarios. Smaller is better.}
\label{tab:sim-rank}
\scalebox{0.8}{
\begin{tabular}{lrrrr}
\toprule
Method & RMSE rank & Bias rank & Coverage rank & Avg. \\
\midrule
DR-FRL-Catoni & 1.50 & 1.17 & 2.92 & 1.86 \\
iTMLE & 3.67 & 3.67 & 4.08 & 3.81 \\
DR-FRL-Mean & 5.50 & 3.67 & 4.67 & 4.61 \\
CRN-style & 3.00 & 3.67 & 8.17 & 4.94 \\
LTMLE & 7.00 & 5.50 & 3.08 & 5.19 \\
Linear SDR & 7.00 & 5.50 & 3.08 & 5.19 \\
Transformer-style & 3.33 & 5.83 & 8.83 & 6.00 \\
Latent-score SDR & 9.17 & 7.50 & 5.83 & 7.50 \\
G-computation & 7.33 & 8.83 & 11.17 & 9.11 \\
IPTW & 12.00 & 12.00 & 4.33 & 9.44 \\
RMSN-style & 8.67 & 10.00 & 10.58 & 9.75 \\
G-Net-style & 9.83 & 10.67 & 11.25 & 10.58 \\
\bottomrule
\end{tabular}}
\end{table}

Overall, the simulations support three conclusions.  First, robust aggregation improves point estimation when functional histories and outcomes produce heavy-tailed pseudo-outcomes.  Second, learned observed-history representations help most when functional signal is high-dimensional or measurement is informative.  Third, overlap diagnostics are essential because weak support makes the distinction between full-population, clipped, and overlap targets practically important.  The population truth is the unclipped intervention value; fixed clipping is therefore finite-sample regularization and a coverage diagnostic, not a change in the reported causal target.

\FloatBarrier

\section{VitalDB Functional-Adjustment and Diagnostic Audit}\label{sec:realdata}

The real-data analysis is a diagnostic audit of functional adjustment, not a claim that an observational perioperative database can by itself validate a treatment policy.  We use VitalDB, the open perioperative database described by \cite{lee2022vitaldb}, which contains 6388 surgical cases, 73 clinical parameters, 34 laboratory time-series parameters, and 486,451 waveform or numeric tracks.  Open physiological databases have made reproducible critical-care and perioperative modeling more practical, but they also expose the irregular-measurement problems that motivate functional point-cloud adjustment \citep{goldberger2000physionet}.  We combine the clinical-information table, the case-level track list, and the laboratory time-series table to construct irregular pre-treatment laboratory point clouds for perioperative risk adjustment.

The application is deliberately narrower than the longitudinal theory.  It is a point-treatment functional-covariate analysis of intraoperative vasopressor exposure and postoperative ICU admission.  Pre-treatment laboratory measurements enter the estimators as variable-size point clouds of observation time, marker, and value.  Intraoperative arterial pressure and heart-rate waveforms are not used to define a sequential target trial here, because treatment timing and waveform-derived histories require a separate track-level validation step.  The purpose is therefore to audit whether irregular pre-treatment functional histories can be incorporated into orthogonal adjustment, whether the resulting contrast has empirical support, and whether the functional representation changes the conclusion relative to transparent scalar summaries.

\subsection{Data and cohort}

The cohort contains all 6388 VitalDB surgical cases with available clinical information and nonzero operation duration.  The clinical table provides demographics, surgery and anesthesia descriptors, preoperative laboratory snapshots, intraoperative drug summaries, and postoperative outcomes.  The track-list table records case-specific monitoring availability, including arterial MAP, non-invasive MAP, heart rate, pulse oximetry, BIS, and drug-rate tracks.  The laboratory time-series table provides irregular measurements before anesthesia induction.

The main functional covariate is the pre-treatment laboratory point cloud $O_i^X=\{(d_{i\ell}, m_{i\ell}, x_{i\ell}): -7\le d_{i\ell}\le 0\}$, where $d_{i\ell}$ is days from anesthesia induction, $m_{i\ell}$ is the marker name, and $x_{i\ell}$ is the measured value.  The markers are hemoglobin, platelet count, creatinine, blood urea nitrogen, albumin, glucose, sodium, potassium, AST, and ALT.  The representation therefore uses irregular pre-treatment values rather than only measurement-availability indicators.  Missingness is handled through model-native median imputation, missing-category handling for categorical variables, and explicit measurement-count features; measurement intensity is retained as part of the observed history.

\begin{table}[H]
\centering
\caption{VitalDB cohort summary with irregular laboratory point clouds and track-availability diagnostics.}
\label{tab:vitaldb-functional-cohort}
\small
\begin{tabular}{lr}
\toprule
Quantity & Value \\
\midrule
Analyzed cases & 6388 \\
Any vasopressor prevalence & 0.543 \\
Phenylephrine prevalence & 0.132 \\
Ephedrine prevalence & 0.503 \\
Postoperative ICU admission & 0.188 \\
ICU stay $>$1 day & 0.061 \\
In-hospital mortality & 0.009 \\
Median pre-op lab points & 10 \\
Cases with any pre-op lab point & 0.778 \\
Median number of VitalDB tracks & 78 \\
Cases with arterial MAP track & 0.583 \\
Cases with NIBP MAP track & 0.902 \\
\bottomrule
\end{tabular}
\end{table}

Table~\ref{tab:vitaldb-functional-cohort} summarizes the cohort.  Vasopressor exposure is common: 54.3\% of cases receive ephedrine, phenylephrine, or epinephrine during surgery.  Postoperative ICU admission occurs in 18.8\% of cases, ICU stay longer than one day in 6.1\%, and in-hospital mortality in 0.9\%.  The median case has 10 preoperative laboratory point observations and 78 VitalDB monitoring tracks.  Arterial MAP and non-invasive MAP tracks are available in 58.3\% and 90.2\% of cases, respectively.  These features motivate functional representation learning and careful overlap diagnostics.

\begin{figure}[H]
\centering
\includegraphics[width=0.8\textwidth]{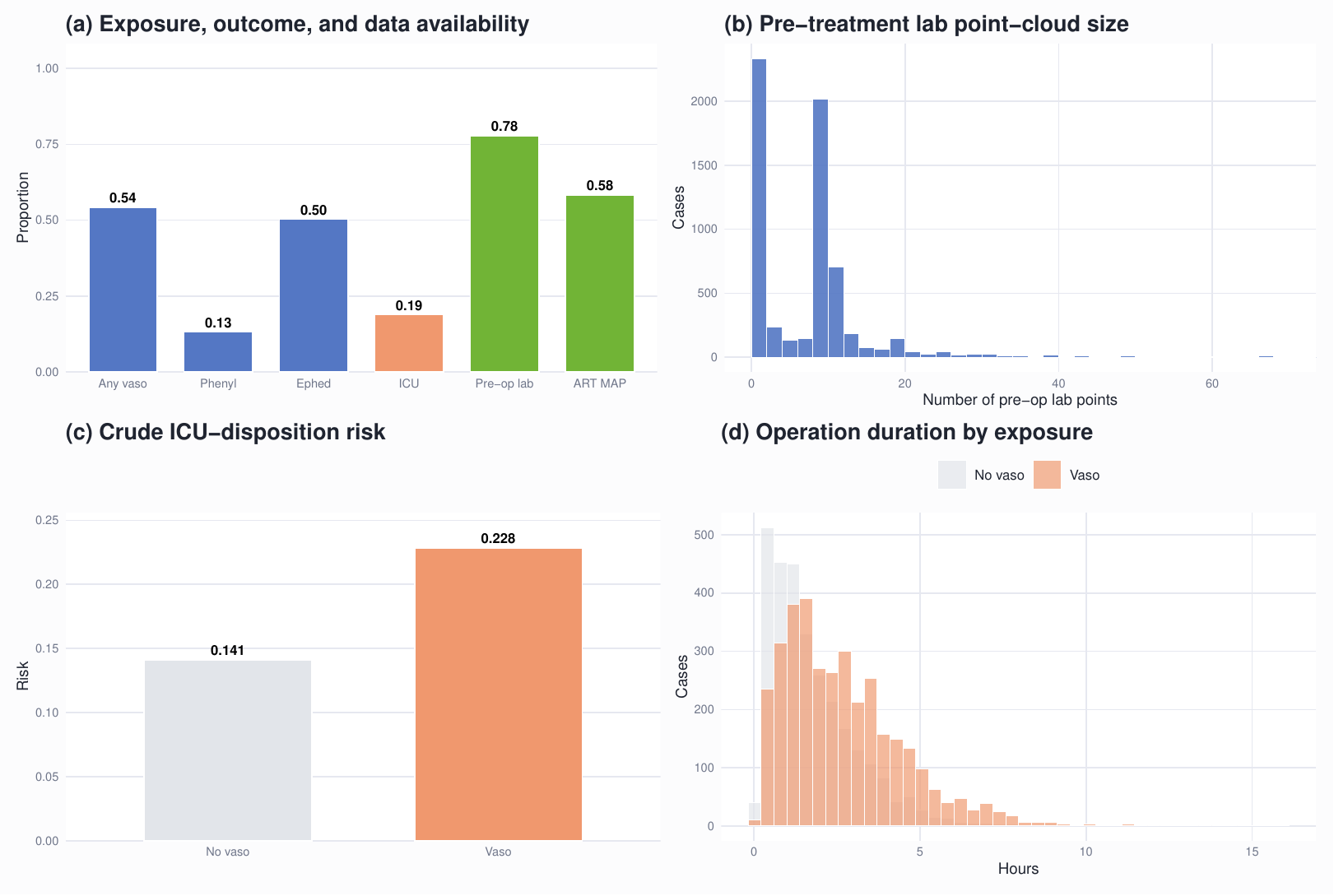}
\caption{VitalDB functional-adjustment dashboard for exposure, outcome, laboratory point-cloud size, crude risk, and operation-duration imbalance.}
\label{fig:vitaldb-functional-dashboard}
\end{figure}

\begin{figure}[H]
\centering
\includegraphics[width=0.8\textwidth]{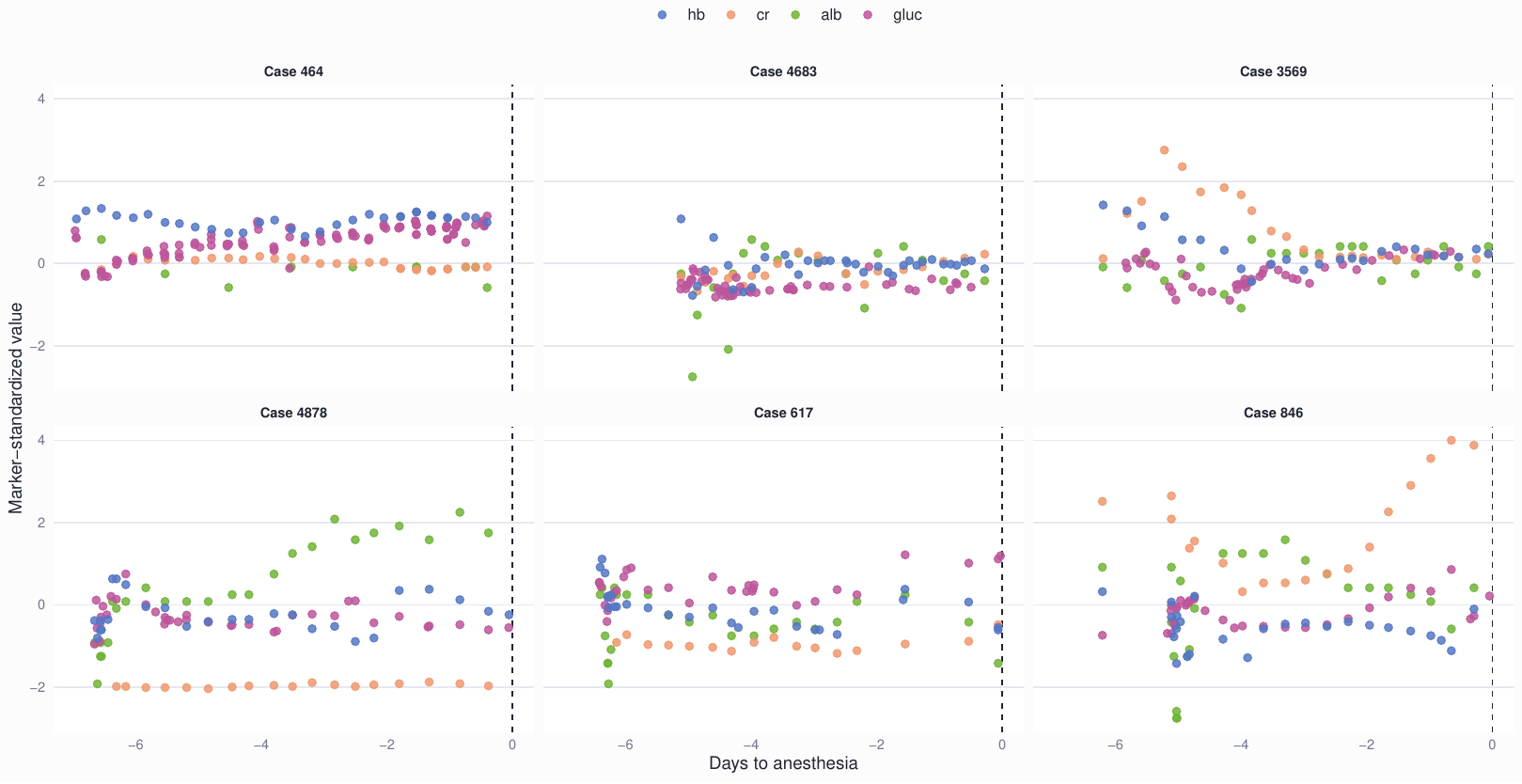}
\caption{Representative irregular pre-treatment laboratory point clouds from VitalDB.  Values are marker-standardized for display.}
\label{fig:vitaldb-lab-pointclouds}
\end{figure}

Figure~\ref{fig:vitaldb-functional-dashboard} gives a multi-panel audit of the analyzed data.  The exposure/outcome panel establishes that vasopressor use and ICU admission are both common enough for estimation; the point-cloud-size panel shows that the laboratory history is irregular rather than a fixed table; the crude-risk panel shows the unadjusted vasopressor--ICU association; and the operation-duration panel reveals case-complexity imbalance that motivates covariate adjustment and overlap assessment.

Figure~\ref{fig:vitaldb-lab-pointclouds} displays representative point clouds.  Each panel is one surgical case; each dot is a pre-treatment laboratory measurement, marker-standardized for display so that timing and measurement density are visible across laboratory types.  The variability in the number, timing, and marker mix of observations is the irregular functional-data structure targeted by DR-FRL.

\subsection{Estimand}

In this application, the observed data are $O=(W,A,Y)$, where $W$ contains baseline clinical variables, surgery and anesthesia descriptors, track-availability summaries, and the pre-treatment laboratory point cloud $O_i^X$; $A$ is a point exposure summary; and $Y$ is a postoperative endpoint.  There are no decision-time-indexed follow-up indicators $R_t$ in this analysis, and no sequential treatment-censoring weights are used.  The estimator is the cross-fitted point-treatment AIPW analogue,
\[
\widehat\theta=\mathbb{P}_n\left[
\widehat m_1(W)-\widehat m_0(W)
+\frac{A\{Y-\widehat m_1(W)\}}{\widehat e(W)}
-\frac{(1-A)\{Y-\widehat m_0(W)\}}{1-\widehat e(W)}
\right],
\]
with propensity clipping and out-of-fold nuisance fitting.

The primary exposure is any intraoperative vasopressor support,
\[
A=\ind\{\texttt{intraop\_eph}>0\ \text{or}\ \texttt{intraop\_phe}>0\ \text{or}\ \texttt{intraop\_epi}>0\}.
\]
The primary endpoint is postoperative ICU admission, $Y=\ind\{\texttt{icu\_days}>0\}$.  We treat this as a pragmatic postoperative disposition and resource-use endpoint rather than as a clean adverse-outcome label, because planned and unplanned ICU admission cannot be separated in the public clinical table.  The estimand is the observed-population adjusted risk difference $\theta=E(Y^1)-E(Y^0)$, under the standard point-treatment identification assumptions conditional on the observed preoperative history.  The AIPW form is a point-treatment analogue of the doubly robust estimating equations of \cite{bang2005doubly}.  This estimand is narrower than the longitudinal dynamic-regime estimand developed in the main methodology.  VitalDB provides reliable intraoperative drug summaries for the point exposure, whereas a multi-interval MAP/vasopressor target trial requires validated track-level drug timing and waveform histories.  The same AIPW/DR-FRL machinery can be extended to such regimes once those track-level histories have been constructed.

Because vasopressor exposure is clinically heterogeneous, we examine any vasopressor, phenylephrine, ephedrine, and high-dose vasopressor use.  Outcome sensitivity analyses consider ICU admission, ICU stay longer than one day, and a composite adverse endpoint.  These analyses separate clinically distinct drug and outcome definitions that would be obscured by a single pooled contrast.

\subsection{Representation ablation}

The main empirical comparison isolates the value of functional information in an audit sense.  We compare six estimators:
\begin{enumerate}[leftmargin=1.8em]
\item scalar clinical covariates only;
\item scalar covariates plus handcrafted laboratory summaries, including counts, means, minima, maxima, last values, and slopes;
\item DR-FRL RBF Mean, using time-kernel embeddings of the irregular laboratory point clouds;
\item DR-FRL RBF Catoni, the same functional representation with Catoni final aggregation;
\item random-forest doubly robust estimation using the functional representation, included because tree-based nuisance learners are common in the heterogeneous-effect estimation framework of \cite{wager2018estimation};
\item linear doubly robust estimation using the functional representation.
\end{enumerate}
The functional RBF encoder is a lightweight set-style encoder: for each laboratory marker and each center in $\{-7,-3,-1,-0.2\}$ days, it records kernel-weighted values and kernel masses.  This representation captures the irregular timing and marker mix of the point cloud while remaining simple enough for transparent ablation.  All AIPW estimators are cross-fitted, propensities are clipped to $[0.025,0.975]$, and uncertainty is computed from out-of-fold pseudo-outcome variability.

\begin{table}[H]
\centering
\caption{VitalDB functional point-cloud analysis of vasopressor support and postoperative ICU admission.}
\label{tab:vitaldb-functional-estimates}
\setlength{\tabcolsep}{3pt}
\resizebox{\linewidth}{!}{
\begin{tabular}{p{0.27\textwidth}rrrrrrrrr}
\toprule
Method & Risk Y1 & Risk Y0 & RD & SE & 95\% CI & AUC & Brier & Min ESS & Clip \% \\
\midrule
Scalar only & 0.193 & 0.181 & 0.013 & 0.008 & [-0.003, 0.028] & 0.721 & 0.212 & 2374.828 & 0.000 \\
Scalar + lab summaries & 0.192 & 0.180 & 0.011 & 0.008 & [-0.003, 0.026] & 0.721 & 0.212 & 2401.140 & 0.000 \\
DR-FRL RBF Mean & 0.192 & 0.181 & 0.010 & 0.008 & [-0.004, 0.025] & 0.720 & 0.212 & 2407.996 & 0.000 \\
DR-FRL RBF Catoni & 0.192 & 0.181 & 0.011 & 0.008 & [-0.004, 0.026] & 0.720 & 0.212 & 2407.996 & 0.000 \\
RF DR functional & 0.202 & 0.165 & 0.038 & 0.008 & [0.023, 0.053] & 0.698 & 0.221 & 2676.653 & 0.000 \\
Linear DR functional & 0.193 & 0.180 & 0.013 & 0.019 & [-0.023, 0.050] & 0.713 & 0.216 & 1368.213 & 1.080 \\
\bottomrule
\end{tabular}}
\end{table}

\begin{figure}[H]
\centering
\includegraphics[width=0.8\textwidth]{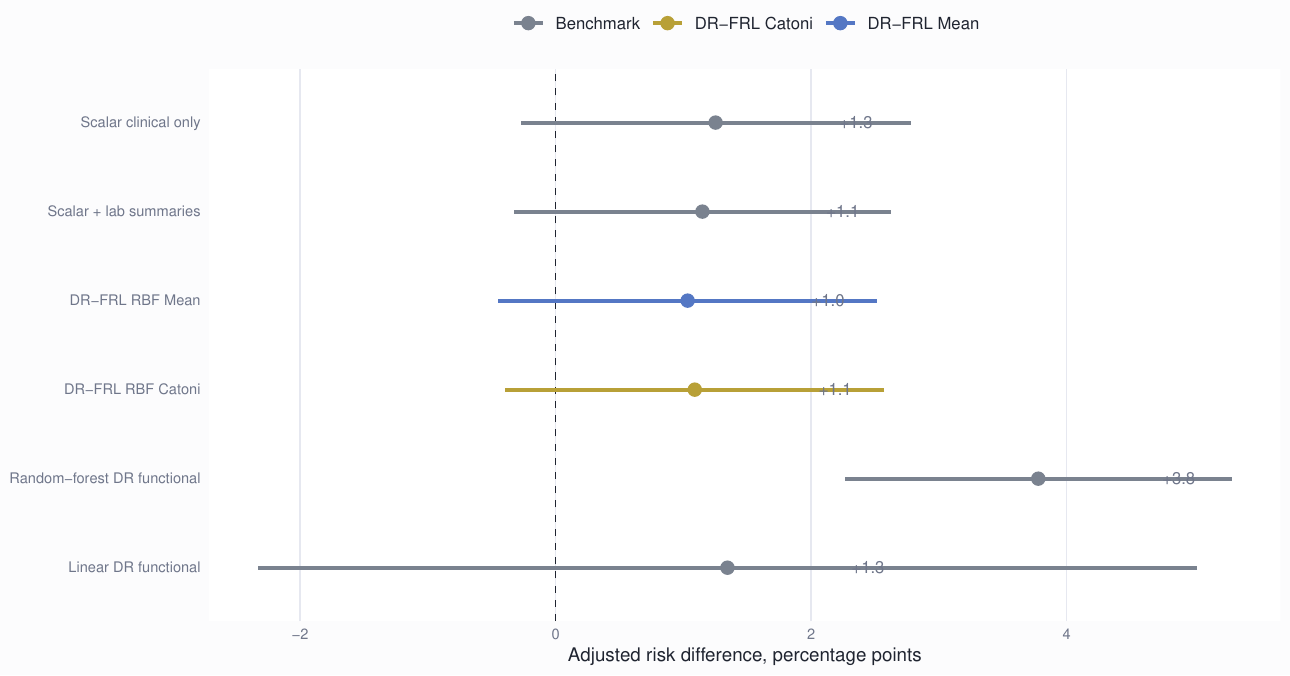}
\caption{Representation ablation and benchmark comparison in the VitalDB diagnostic audit.  Points are adjusted risk differences and bars are 95\% intervals.}
\label{fig:vitaldb-functional-forest}
\end{figure}

Table~\ref{tab:vitaldb-functional-estimates} reports treatment-specific risks, adjusted risk differences, nuisance diagnostics, effective sample sizes, and clipping fractions.  The scalar-only adjusted risk difference is 0.0125.  Adding handcrafted laboratory summaries gives 0.0115; the functional RBF estimator gives 0.0103; Catoni aggregation gives 0.0109.  The estimates cluster near one percentage point, so the functional point cloud does not materially change the primary contrast after rich scalar clinical adjustment.  This is a useful negative finding: DR-FRL can add irregular point-cloud histories to orthogonal adjustment and audit support, but it need not force a functional effect when scalar summaries already carry the endpoint-relevant information.  The random-forest DR estimate is larger, 0.0378, and should be interpreted alongside its nuisance diagnostics.

Figure~\ref{fig:vitaldb-functional-forest} visualizes the same ablation.  The treatment-specific risks are clinically useful: for the functional DR-FRL estimator, the estimated ICU-admission risk is 0.192 under vasopressor support and 0.181 under no support, corresponding to an adjusted risk difference of approximately 1.0 percentage point.

\subsection{Overlap and primary diagnostics}

The common-support diagnostics are favorable for the primary functional estimator.  The primary propensity model has AUC 0.720 and Brier score 0.212.  The minimum effective sample size is approximately 2,408 and the clip fraction is zero at the 0.025 threshold.  Figure~\ref{fig:vitaldb-overlap-calibration} shows that treated and untreated cases have overlapping predicted propensities, and the calibration curve does not reveal a severe exposure-model failure.  Reporting these diagnostics is important because, as discussed by \cite{petersen2012positivity}, lack of overlap changes the estimand that can be supported by the data.

\begin{figure}[H]
\centering
\includegraphics[width=0.96\textwidth]{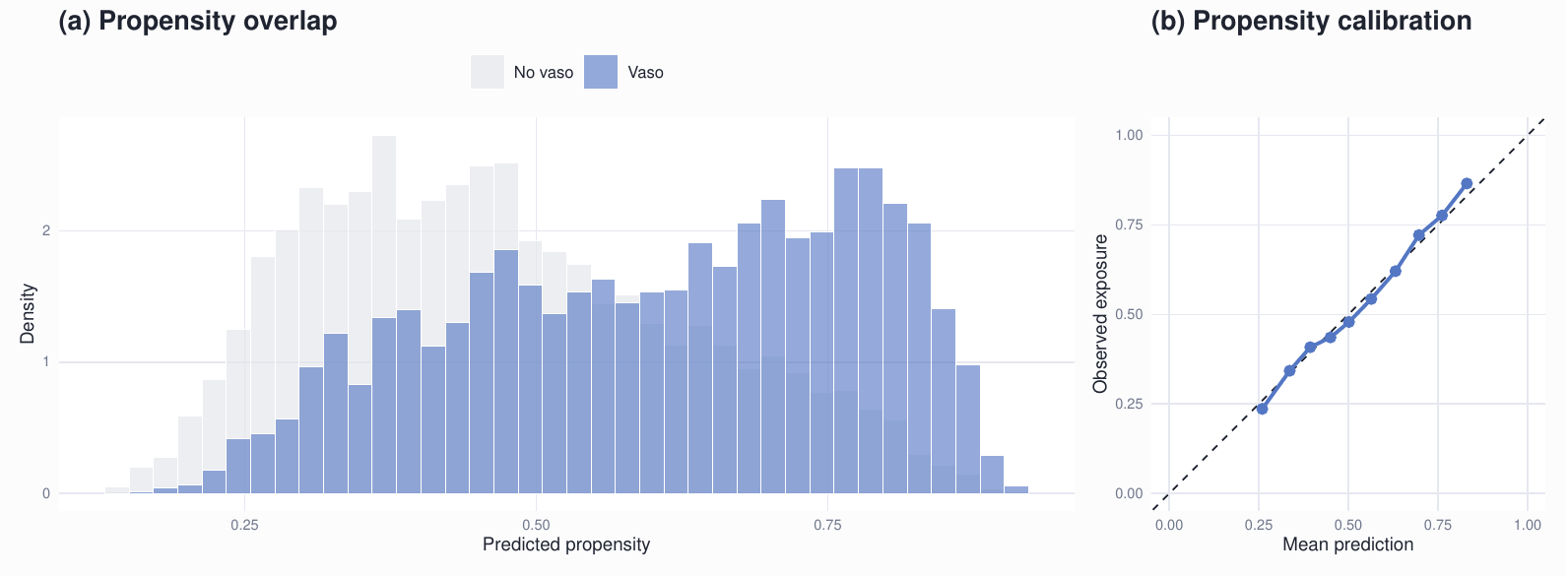}
\caption{Overlap and calibration diagnostics for the functional DR-FRL analysis.}
\label{fig:vitaldb-overlap-calibration}
\end{figure}

\begin{table}[H]
\centering
\caption{Common-support trimming sensitivity for the functional DR-FRL estimator.}
\label{tab:vitaldb-trimming}
\small
\begin{tabular}{rrrrrrrr}
\toprule
Threshold & Risk Y1 & Risk Y0 & RD & SE & 95\% CI & Min ESS & Clip \% \\
\midrule
0.010 & 0.192 & 0.181 & 0.010 & 0.008 & [-0.004, 0.025] & 2407.996 & 0.000 \\
0.025 & 0.192 & 0.181 & 0.010 & 0.008 & [-0.004, 0.025] & 2407.996 & 0.000 \\
0.050 & 0.192 & 0.181 & 0.010 & 0.008 & [-0.004, 0.025] & 2407.996 & 0.000 \\
0.100 & 0.192 & 0.181 & 0.010 & 0.008 & [-0.004, 0.025] & 2407.996 & 0.078 \\
\bottomrule
\end{tabular}
\end{table}

Table~\ref{tab:vitaldb-trimming} reports trimming sensitivity.  Thresholds 0.01, 0.025, and 0.05 give the same estimate because no predicted propensity falls outside those bounds.  At threshold 0.10, only 0.08\% of cases are clipped and the estimate remains essentially unchanged.  This supports reporting the observed-population risk difference rather than switching to an overlap-only estimand in this particular analysis.

Additional balance, marker-missingness, exposure/outcome sensitivity, and pseudo-outcome tail diagnostics are reported in Supplementary Section S4.  These diagnostics support the same interpretation as the main ablation: common support is acceptable for the primary point-treatment contrast, but the analysis remains observational and cannot rule out unmeasured confounding by surgical complexity, planned postoperative monitoring, or clinician decision-making.

\subsection{Limitations}

The VitalDB analysis has several limitations.  The estimand is a point-treatment contrast based on intraoperative drug summaries, not a full longitudinal dynamic-regime analysis.  Planned versus unplanned ICU admission cannot be cleanly separated from the available public clinical table, so postoperative ICU admission is a pragmatic resource-use endpoint rather than a definitive clinical causal outcome.  The analysis also cannot eliminate unmeasured confounding by surgical complexity, anesthesiologist decision-making, or planned postoperative monitoring.

The empirical conclusion is accordingly cautious.  DR-FRL-style point-cloud adjustment can be applied to open perioperative patient-level data, and the adjusted vasopressor--ICU contrast is much smaller than the crude contrast.  The functional representation does not improve materially over carefully engineered laboratory summaries for this endpoint, but it gives a coherent route for incorporating irregular measurements into orthogonal causal estimation.  A full intraoperative MAP/vasopressor target trial would require validated track-level treatment timing and waveform histories.

\section{Discussion}\label{sec:discussion}

DR-FRL is built on a simple principle: reduce functional histories only in ways that preserve the nuisance information needed by the causal estimand.  A useful representation is not the one that best reconstructs a curve; it is the one that stabilizes the EIF, overlap, calibration, and pseudo-outcome distribution.

The method links flexible functional encoders with an orthogonal longitudinal estimating equation, extending the targeted-learning idea of \cite{schuler2017targeted} to irregular histories.  This link has three benefits: point-cloud histories can be used without naive gridding, representation error appears through an explicit product remainder, and the pseudo-outcome recursion yields overlap, calibration, tail, effective-sample-size, and ablation diagnostics.

The simulations show where this matters.  DR-FRL is most useful when functional signals are hard to summarize, measurement is informative, overlap is weak, or pseudo-outcomes are heavy-tailed.  Catoni aggregation improves point accuracy in these stress tests; formal Wald inference remains tied to mean aggregation.

The VitalDB analysis illustrates the method as an audit.  Pre-treatment laboratory point clouds can be included in orthogonal adjustment, and the crude vasopressor--ICU association is much smaller after adjustment.  Functional RBF estimates are close to scalar-summary estimates, suggesting that for this endpoint scalar laboratory summaries already carry much of the relevant information.

Important limitations remain.  The observed-history target does not remove unmeasured confounding by latent functional features that are not observed or recoverable.  Clipped and overlap targets answer different questions from the full observed-population target.  Robust aggregation limits the effect of heavy-tailed pseudo-outcomes but cannot repair identification failures.  The VitalDB analysis is observational and point-treatment based; a full intraoperative dynamic-regime study would require validated treatment timing, waveform histories, and outcome definitions.

\section*{Disclosure Statement}
No potential conflict of interest was reported by the authors.


\bibliographystyle{erae}
\bibliography{references}

\appendix

\newpage
\section{Proofs}\label{app:proofs}

This section gives the full proof details for main-text Theorems~\ref{thm:identification}--\ref{thm:robust}, Proposition~\ref{prop:selector}, Lemma~\ref{lem:projection-stability}, and Corollary~\ref{cor:sieve}.  Throughout the supplement, $P$ denotes $P_0$, $\Pn$ denotes the empirical measure, and the target regime $g$ is fixed.  Static regimes are obtained by setting $g_t(h)=a_t$.  Write
\[
\Gamma_t^g=(\pi_t^g,\rho_t^g),\qquad
G_t^g(H_t)=\prod_{s=0}^t\pi_s^g(H_s)\rho_s^g(H_s),
\]
and let
\[
\mathcal H_t^g(\eta)=
\frac{\ind\{A_s=g_s(H_s),R_s=1,\text{ for all }s\le t\}}
{\prod_{s=0}^t\pi_s^g(H_s)\rho_s^g(H_s)}
\]
be the cumulative clever covariate associated with a working nuisance vector $\eta=(Q,\pi,\rho)$.  The true nuisance vector is $\eta^\star=(Q^\star,\pi^\star,\rho^\star)$.  Define the EIF map
\[
D_g(O;\eta,\theta)=Q_0^g(H_0)-\theta+
\sum_{t=0}^T\mathcal H_t^g(\eta)
\{Q_{t+1}^g(H_{t+1})-Q_t^g(H_t)\},
\]
where $Q_{T+1}^g(H_{T+1})=Y$.  At the truth, $D_g(O;\eta^\star,\theta(g))=\phi_g^\star(O)$.

Cross-fitting is interpreted fold by fold throughout.  Let $\mathcal I_1,\ldots,\mathcal I_K$ be the evaluation folds and let $\mathcal T_k$ be the sigma-field generated by all observations outside $\mathcal I_k$, together with the random splitting, representation-selection decisions, calibration choices, early-stopping decisions, and optional targeting steps made inside the training sample.  The nuisance vector used for $i\in\mathcal I_k$ is written $\widehat\eta^{(-k)}=(\widehat Q^{(-k)},\widehat\pi^{(-k)},\widehat\rho^{(-k)})$, and the corresponding estimated influence value is written $\widehat\phi_i^{(-k)}=D_g(O_i;\widehat\eta^{(-k)},\widehat\theta_g)$.  Conditional on $\mathcal T_k$, the evaluation observations $\{O_i:i\in\mathcal I_k\}$ are independent of $\widehat\eta^{(-k)}$.  All empirical-process and variance statements below are therefore first read conditionally on $\mathcal T_k$ and then summed over the fixed number of folds.

\subsection{Two algebraic lemmas}\label{app:lemmas}

\begin{drfdllemma}[Telescoping identity]\label{lem:telescoping}
For any working nuisance vector $\eta=(Q,\pi,\rho)$, the backward recursion
\[
\widetilde Y_{T+1}^g=Y,\qquad
\widetilde Y_t^g=Q_t^g(H_t)+\omega_t^g(\eta)\{\widetilde Y_{t+1}^g-Q_t^g(H_t)\}
\]
satisfies
\begin{equation}
\widetilde Y_t^g=Q_t^g(H_t)+
\sum_{s=t}^T\left\{\prod_{j=t}^s\omega_j^g(\eta)\right\}
\{Q_{s+1}^g(H_{s+1})-Q_s^g(H_s)\}.
\label{eq:app-telescope-t}
\end{equation}
In particular, $\widetilde Y_0^g-\theta=D_g(O;\eta,\theta)$.
\end{drfdllemma}

\begin{proof}
For $t=T$, the recursion gives
\[
\widetilde Y_T^g=Q_T^g(H_T)+\omega_T^g(\eta)\{Y-Q_T^g(H_T)\},
\]
which is exactly \eqref{eq:app-telescope-t} because $Q_{T+1}^g(H_{T+1})=Y$.  Suppose \eqref{eq:app-telescope-t} holds for $t+1$.  Substituting the induction hypothesis into the recursion at time $t$ gives
\begin{align*}
\widetilde Y_t^g
&=Q_t^g(H_t)+\omega_t^g(\eta)
\left[Q_{t+1}^g(H_{t+1})-Q_t^g(H_t)\right.\\
&\hspace{4em}\left.+\sum_{s=t+1}^T
\left\{\prod_{j=t+1}^s\omega_j^g(\eta)\right\}
\{Q_{s+1}^g(H_{s+1})-Q_s^g(H_s)\}\right]\\
&=Q_t^g(H_t)+\sum_{s=t}^T
\left\{\prod_{j=t}^s\omega_j^g(\eta)\right\}
\{Q_{s+1}^g(H_{s+1})-Q_s^g(H_s)\}.
\end{align*}
Thus the result follows by backward induction.  For $t=0$, the product of incremental clever covariates is the cumulative clever covariate $\mathcal H_s^g(\eta)$, which yields $\widetilde Y_0^g-\theta=D_g(O;\eta,\theta)$.
\end{proof}

\begin{drfdllemma}[Conditional recursion identity]\label{lem:conditional}
For any working nuisance vector $\eta$, define the mechanism ratio
\[
\lambda_t(H_t;\eta)=
\frac{\pi_t^{g,\star}(H_t)\rho_t^{g,\star}(H_t)}
{\pi_t^g(H_t)\rho_t^g(H_t)},
\]
where the numerator uses the true treatment and incremental censoring probabilities and the denominator uses the working probabilities.  All conditional expectations below are on the at-risk event $\bar R_{t-1}=1$, which is suppressed to keep notation readable.  Then
\begin{align*}
E\{\widetilde Y_t^g(\eta)\mid H_t\}
&=Q_t^g(H_t)+\lambda_t(H_t;\eta)
\big[E\{\widetilde Y_{t+1}^g(\eta)\mid H_t,A_t=g_t(H_t),R_t=1\}
-Q_t^g(H_t)\big].
\end{align*}
\end{drfdllemma}

\begin{proof}
Condition on $H_t$.  The quantity $Q_t^g(H_t)$ and the denominator of $\omega_t^g(\eta)$ are fixed.  Moreover
\[
E\{\ind(A_t=g_t(H_t),R_t=1)\mid H_t\}
=\pi_t^{g,\star}(H_t)\rho_t^{g,\star}(H_t),
\]
where the censoring probability is conditional on $(H_t,A_t=g_t(H_t))$.  Applying iterated expectation to the product of the indicator and the future pseudo-outcome gives
\begin{align*}
E\{\widetilde Y_t^g\mid H_t\}
&=Q_t^g(H_t)+
\frac{E[\ind\{A_t=g_t(H_t),R_t=1\}
\{\widetilde Y_{t+1}^g-Q_t^g(H_t)\}\mid H_t]}
{\pi_t^g(H_t)\rho_t^g(H_t)}\\
&=Q_t^g(H_t)+
\frac{\pi_t^{g,\star}(H_t)\rho_t^{g,\star}(H_t)}
{\pi_t^g(H_t)\rho_t^g(H_t)}
\big[E\{\widetilde Y_{t+1}^g\mid H_t,A_t=g_t(H_t),R_t=1\}-Q_t^g(H_t)\big].
\end{align*}
This proves the identity.  If $R_t=0$, the next history is the cemetery value and the indicator in $\omega_t^g$ is zero, so the arbitrary value assigned to $Q_{t+1}^g(\partial)$ never enters the conditional expectation.
\end{proof}

\subsection{Finite-library selection}\label{app:proof-selector}

\begin{proof}
Let $\widehat\phi$ be an $\epsilon_n$-approximate minimizer of $\widehat R_{\mathrm{EIF}}$ over $\Phi_n$, and let $\phi^\circ\in\argmin_{\phi\in\Phi_n}R_{\mathrm{EIF}}(\phi)$ be a population oracle candidate.  Then
\begin{align*}
R_{\mathrm{EIF}}(\widehat\phi)
&\le \widehat R_{\mathrm{EIF}}(\widehat\phi)
+\sup_{\phi\in\Phi_n}|R_{\mathrm{EIF}}(\phi)-\widehat R_{\mathrm{EIF}}(\phi)|\\
&\le \widehat R_{\mathrm{EIF}}(\phi^\circ)+\epsilon_n
+\sup_{\phi\in\Phi_n}|R_{\mathrm{EIF}}(\phi)-\widehat R_{\mathrm{EIF}}(\phi)|\\
&\le R_{\mathrm{EIF}}(\phi^\circ)+\epsilon_n
+2\sup_{\phi\in\Phi_n}|R_{\mathrm{EIF}}(\phi)-\widehat R_{\mathrm{EIF}}(\phi)|.
\end{align*}
The assumed uniform validation concentration gives the stated bound.  The concentration condition can be verified by Bernstein or bounded-difference inequalities when the validation EIF scores are conditionally sub-exponential and the candidate library is finite, followed by a union bound over $M_n$ candidates.  Because the oracle inequality concerns only the validation EIF criterion, it does not imply that the oracle projection errors $b_{Q,t}$ and $b_{\Gamma,t}$ are small unless the library contains a representation with that property.
\end{proof}

\subsection{EIF-map stability}\label{app:proof-projection-stability}

\begin{proof}
For the outcome nuisance, add and subtract the best representation-indexed approximation:
\begin{align*}
\norm{\widehat Q_t-Q_t^{g,\star}}_2
&\le \norm{\widehat Q_t-Q_{t,\widehat\phi}^{\dagger}}_2
+\norm{Q_{t,\widehat\phi}^{\dagger}-Q_t^{g,\star}}_2
=r_{Q,t}+b_{Q,t}(\widehat\phi).
\end{align*}
The same triangle inequality for the vector mechanism $\Gamma_t=(\pi_t,\rho_t)$ gives
\[
\norm{\widehat\Gamma_t-\Gamma_t^{g,\star}}_2
\le r_{\Gamma,t}+b_{\Gamma,t}(\widehat\phi),
\]
where the vector norm is the Euclidean norm inside the outer $L_2(P_0)$ norm.

It remains to show how these total-error bounds control the empirical-process class used in main-text Theorem~3.  Let $e_{Q,t}=\norm{\widehat Q_t-Q_t^{g,\star}}_2$ and $e_{\Gamma,t}=\norm{\widehat\Gamma_t-\Gamma_t^{g,\star}}_2$.  Under the lower probability bound or deterministic clipping, the map $(\pi_t,\rho_t)\mapsto\{\pi_t\rho_t\}^{-1}$ is locally Lipschitz on the relevant support.  Therefore the cumulative clever covariate difference satisfies, for fixed $T$,
\[
\norm{\widehat{\mathcal H}_t^g-\mathcal H_t^{g,\star}}_2
\le C_T\sum_{s=0}^t e_{\Gamma,s}+o_P(1),
\]
with $C_T$ depending on the lower bound or clipping level.  Writing $M_t^\star=Q_{t+1}^{g,\star}-Q_t^{g,\star}$ and $\widehat M_t=\widehat Q_{t+1}^{g}-\widehat Q_t^{g}$, the difference of EIF maps is a sum of terms of the form
\[
(\widehat{\mathcal H}_t^g-\mathcal H_t^{g,\star})M_t^\star
+\widehat{\mathcal H}_t^g(\widehat M_t-M_t^\star)
+\widehat Q_0^g-Q_0^{g,\star}.
\]
Cauchy-Schwarz, bounded or clipped weights, and bounded second moments of $M_t^\star$ give
\[
\norm{D_g(\cdot;\widehat\eta,\theta(g))-D_g(\cdot;\eta^\star,\theta(g))}_2
\le C_T\sum_{t=0}^T(e_{Q,t}+e_{\Gamma,t})+o_P(1).
\]
Substituting the projection bounds above shows that the right-hand side is $o_P(1)$ whenever the total stochastic and approximation errors vanish.  This proves main-text Assumption~7 under the stated primitive conditions.
\end{proof}

\subsection{Identification and EIF}\label{app:proof-identification}

\begin{proof}
We first prove identification for the observed-history target.  By consistency, among individuals whose observed treatment path agrees with $g$ and who remain under follow-up, the observed outcome equals the corresponding counterfactual outcome.  For a dynamic regime, the exchangeability object must include the future histories that determine later regime decisions.  Thus, for each time $t$, write
\[
\mathcal F_t^g=(Y^g,H_{t+1}^g,\ldots,H_{T+1}^g).
\]
By sequential exchangeability, conditional on $H_t$ and being under follow-up just before $t$, treatment assignment is independent of $\mathcal F_t^g$; conditional on $(H_t,A_t)$, the follow-up indicator is also independent of $\mathcal F_t^g$.  Positivity ensures that the conditional distributions required under the intervention are observed.

The recursion is defined on the at-risk transition.  More precisely, $H_{t+1}$ denotes the next pre-treatment history generated after observing $(H_t,A_t,R_t)$ on the event $R_t=1$.  If $R_t=0$, we set $H_{t+1}=\partial$ and $Q_{t+1}^{g,\star}(\partial)$ arbitrarily.  This convention makes all random variables measurable on a common sample space, while the factor $\ind(R_t=1)$ in the clever covariate ensures that the arbitrary cemetery value is never used in the EIF.  At the terminal time,
\[
Q_T^{g,\star}(h_T)
=E\{Y\mid H_T=h_T,A_T=g_T(h_T),R_T=1,\bar R_{T-1}=1\}.
\]
At the previous time point,
\[
Q_{T-1}^{g,\star}(h_{T-1})
=E\{Q_T^{g,\star}(H_T)\mid H_{T-1}=h_{T-1},A_{T-1}=g_{T-1}(h_{T-1}),R_{T-1}=1,\bar R_{T-2}=1\}.
\]
Repeating the argument backward proves main-text equation~(4).  Taking expectation over $H_0$ yields $\theta(g)=E\{Q_0^{g,\star}(H_0)\}$.

For a latent full-functional-history target, the preceding argument would take place in a full-data law.  Such a target is not identified by the present observed likelihood unless a separate bridge or measurement model recovers the relevant full-data sequential regressions from the coarsened histories.  Even under such a bridge, the observed-data canonical gradient is the projection of the full-data gradient onto the observed-data tangent space.  It coincides with main-text equation~(6) only when the observation-process score is orthogonal to the target, so that no additional inverse-coarsening or coarsening-score term appears.  Otherwise the present article does not claim the full-functional target; it targets the observed-history value.

We next derive the EIF for the observed-history target.  Consider a regular parametric submodel $P_\varepsilon$ through $P$ with score $S(O)$ at $\varepsilon=0$.  The score decomposes into components for the baseline history, the conditional law of each next pre-treatment history on the at-risk regime path, the treatment mechanisms, the censoring mechanisms, and the terminal outcome law.  The target functional is $\theta_\varepsilon(g)=E_\varepsilon\{Q_{0,\varepsilon}^g(H_0)\}$.  The derivative of the baseline law contributes $Q_0^{g,\star}(H_0)-\theta(g)$.  For each $t$, the derivative of the conditional law of $H_{t+1}$, including the outcome law for $t=T$, contributes the residual
\[
M_t^{g,\star}=Q_{t+1}^{g,\star}(H_{t+1})-Q_t^{g,\star}(H_t)
\]
weighted by the likelihood ratio that maps observed paths to the regime path.  Since $Q_t^{g,\star}$ is exactly the conditional mean of $Q_{t+1}^{g,\star}(H_{t+1})$ given $(H_t,A_t=g_t(H_t),R_t=1,\bar R_{t-1}=1)$, 
\[
E\{M_t^{g,\star}\mid H_t,A_t=g_t(H_t),R_t=1,\bar R_{t-1}=1\}=0.
\]
Therefore the derivative contribution from time $t$ is represented by $\mathcal H_t^{g,\star}M_t^{g,\star}$.  Summing contributions gives
\[
\left.\frac{d}{d\varepsilon}\right|_{0}\theta_\varepsilon(g)
=E\left[\left\{Q_0^{g,\star}(H_0)-\theta(g)+
\sum_{t=0}^T\mathcal H_t^{g,\star}M_t^{g,\star}\right\}S(O)\right].
\]
The displayed candidate has mean zero.  For each $t$,
\[
E\{\mathcal H_t^{g,\star}M_t^{g,\star}\}
=E\left[\mathcal H_{t-1}^{g,\star}E\left\{\frac{\ind(A_t=g_t(H_t),R_t=1)}
{\pi_t^{g,\star}(H_t)\rho_t^{g,\star}(H_t)}M_t^{g,\star}\mid H_t,\bar R_{t-1}=1\right\}\right]=0,
\]
where the inner conditional expectation uses the at-risk transition and the residual mean-zero identity above.  Treatment- and censoring-score components are orthogonal by conditioning on $H_t$ and $(H_t,A_t)$, respectively.  Because the observed-history model is otherwise nonparametric, the candidate is the canonical gradient, proving main-text equation~(6).
\end{proof}

\subsection{Sequential robustness}\label{app:proof-sdr}

\begin{proof}
For a working nuisance vector $\eta$, define
\[
\Delta_t(H_t)=E\{\widetilde Y_t^g(\eta)\mid H_t\}-Q_t^{g,\star}(H_t),\qquad
\Delta_{T+1}=0.
\]
By Lemma~\ref{lem:conditional},
\begin{align*}
E\{\widetilde Y_t^g\mid H_t\}
&=Q_t^g(H_t)+\lambda_t(H_t;\eta)
\{E(\widetilde Y_{t+1}^g\mid H_t,A_t=g_t(H_t),R_t=1)-Q_t^g(H_t)\}.
\end{align*}
Using the definition of $Q_t^{g,\star}$ and adding/subtracting $Q_{t+1}^{g,\star}$ inside the conditional expectation gives
\begin{align}
\Delta_t(H_t)
&=\{1-\lambda_t(H_t;\eta)\}\{Q_t^g(H_t)-Q_t^{g,\star}(H_t)\}
\nonumber\\
&\quad+\lambda_t(H_t;\eta)
E\{\Delta_{t+1}(H_{t+1})\mid H_t,A_t=g_t(H_t),R_t=1\}.
\label{eq:app-delta-recursion}
\end{align}
This identity is exact.

We now prove the result by backward induction.  At $t=T$, $\Delta_{T+1}=0$, so
\[
\Delta_T(H_T)=\{1-\lambda_T(H_T;\eta)\}
\{Q_T^g(H_T)-Q_T^{g,\star}(H_T)\}.
\]
If either $Q_T^g=Q_T^{g,\star}$ or $(\pi_T^g,\rho_T^g)=(\pi_T^{g,\star},\rho_T^{g,\star})$, then $\Delta_T=0$.  Suppose $\Delta_{t+1}=0$ under the robustness conditions for future times.  Equation~\eqref{eq:app-delta-recursion} reduces to the same product form at time $t$.  If $Q_t^g=Q_t^{g,\star}$, the product is zero; if the time-$t$ mechanism is correct, then $\lambda_t=1$ and the product is again zero.  Hence $\Delta_t=0$.  By backward induction, $\Delta_0=0$ and therefore
\[
E\{\widetilde Y_0^g(\eta)\}=E\{Q_0^{g,\star}(H_0)\}=\theta(g).
\]
For cross-fitted estimators whose probability limits satisfy the same time-specific conditions, the law of large numbers applied to the out-of-fold pseudo-outcomes gives consistency.
\end{proof}

\subsection{Representation remainder}\label{app:proof-remainder}

\begin{proof}
We prove the result for the mean aggregation estimator.  The robust-aggregation transfer is handled in main-text Theorem~6.  For $i\in\mathcal I_k$, all fitted nuisance functions are $\widehat\eta^{(-k)}$ and are measurable with respect to $\mathcal T_k$.  Conditional on $\mathcal T_k$, the evaluation observations in $\mathcal I_k$ are independent of the nuisance functions used to score them.  The proof below is written for one generic fold and then summed over the fixed number of folds.

By Lemma~\ref{lem:telescoping}, the feasible estimator can be written as
\[
\widehat\theta_g=\Pn\{D_g(O;\widehat\eta,0)\},
\]
where $\widehat\eta$ denotes $\widehat\eta^{(-k(i))}$ for observation $i$.  Since $P\{D_g(O;\eta^\star,\theta(g))\}=0$, we decompose
\begin{align}
\widehat\theta_g-\theta(g)
&=(\Pn-P)D_g(O;\eta^\star,\theta(g))
\nonumber\\
&\quad+P\{D_g(O;\widehat\eta,\theta(g))-D_g(O;\eta^\star,\theta(g))\}
\nonumber\\
&\quad+(\Pn-P)\{D_g(O;\widehat\eta,\theta(g))-D_g(O;\eta^\star,\theta(g))\}.
\label{eq:app-main-decomp}
\end{align}
The first term is $(\Pn-P)\phi_g^\star$.

We next bound the mean remainder.  Let
\[
\delta_{Q,t}=\widehat Q_t^g-Q_t^{g,\star},\qquad
\delta_{\Gamma,t}=\widehat\Gamma_t^g-\Gamma_t^{g,\star}.
\]
The exact recursion \eqref{eq:app-delta-recursion}, applied with $\eta=\widehat\eta$, implies
\[
\left|P\{D_g(O;\widehat\eta,\theta(g))-D_g(O;\eta^\star,\theta(g))\}\right|
=|E\{\Delta_0(H_0)\}|.
\]
Let $\Gamma_t^g=(\pi_t^g,\rho_t^g)$ and use the product-mechanism norm
\[
\|\widehat\Gamma_t^g-\Gamma_t^{g,\star}\|_2
=\|\widehat\pi_t^g-\pi_t^{g,\star}\|_{L_2(P_0)}
+\|\widehat\rho_t^g-\rho_t^{g,\star}\|_{L_2(P_0)}.
\]
Assume that true and working treatment and censoring probabilities are bounded below by $\varepsilon>0$ after calibration or clipping; this is the lower probability bound referred to in the theorem.  Since
\[
\widehat\lambda_t(H_t)=
\frac{\pi_t^{g,\star}(H_t)\rho_t^{g,\star}(H_t)}
{\widehat\pi_t^g(H_t)\widehat\rho_t^g(H_t)},
\]
the map $(\pi,\rho)\mapsto \pi\rho/(\widehat\pi\widehat\rho)$ is Lipschitz on $[\varepsilon,1]^4$.  Hence there is a constant $C=C(\varepsilon)$ such that
\[
\|1-\widehat\lambda_t\|_2
\le C\|\widehat\Gamma_t^g-\Gamma_t^{g,\star}\|_2.
\]
To make the iteration explicit, let $P_g$ denote the law of the observed histories generated along the regime path using the observed-history transition law.  Positivity and fixed $T$ imply that the relevant $L_2(P_g)$ and $L_2(P_0)$ norms are equivalent up to constants depending only on $\varepsilon$ and $T$.  Define
\[
a_t=E_g\{|\Delta_t(H_t)|\},\qquad a_{T+1}=0,
\]
where $E_g$ is expectation under $P_g$.  Applying \eqref{eq:app-delta-recursion}, Cauchy-Schwarz, the lower probability bound, and norm equivalence gives
\begin{align*}
a_t
&\le E_g\{|1-\widehat\lambda_t(H_t)|\,|\widehat Q_t^g(H_t)-Q_t^{g,\star}(H_t)|\}\\
&\quad+E_g\{|\widehat\lambda_t(H_t)|\,E_g(|\Delta_{t+1}(H_{t+1})|\mid H_t)\}\\
&\le C\|\widehat\Gamma_t^g-\Gamma_t^{g,\star}\|_{2,P_0}
       \|\widehat Q_t^g-Q_t^{g,\star}\|_{2,P_0}
   +C a_{t+1}.
\end{align*}
Iterating this inequality from $T$ to $0$ yields, for fixed $T$,
\begin{equation}
|E\{\Delta_0(H_0)\}|
\le C\sum_{t=0}^T
\|\widehat Q_t^g-Q_t^{g,\star}\|_2
\|\widehat\Gamma_t^g-\Gamma_t^{g,\star}\|_2.
\label{eq:app-product-bound-total}
\end{equation}
This is the longitudinal Neyman-orthogonal product remainder.  It says that a first-order error in $Q_t$ alone or in $\Gamma_t$ alone does not produce first-order target bias; the bias is driven by their products.

The learned representation separates each total nuisance error into stochastic estimation and approximation components.  Let $Q_{t,\widehat\phi}^{\dagger}$ and $\Gamma_{t,\widehat\phi}^{\dagger}$ be the best representation-indexed approximations defined in main-text Assumption~6.  Then
\begin{align*}
\|\widehat Q_t^g-Q_t^{g,\star}\|_2
&\le \|\widehat Q_t^g-Q_{t,\widehat\phi}^{\dagger}\|_2
+\|Q_{t,\widehat\phi}^{\dagger}-Q_t^{g,\star}\|_2
\le r_{Q,t}+b_{Q,t}(\widehat\phi),\\
\|\widehat\Gamma_t^g-\Gamma_t^{g,\star}\|_2
&\le \|\widehat\Gamma_t^g-\Gamma_{t,\widehat\phi}^{\dagger}\|_2
+\|\Gamma_{t,\widehat\phi}^{\dagger}-\Gamma_t^{g,\star}\|_2
\le r_{\Gamma,t}+b_{\Gamma,t}(\widehat\phi).
\end{align*}
Substituting these inequalities into \eqref{eq:app-product-bound-total} gives the bound in main-text equation~(14).  This is the representation-error part of the theorem: representation loss is harmless only when it enters through this product-rate structure.

Finally, consider the empirical-process term in \eqref{eq:app-main-decomp}.  Conditional on the training folds, each evaluation fold is independent of the nuisance functions trained outside that fold.  For a fixed number of folds $K$, the variance of a fold-wise empirical process after multiplication by the fold size is bounded by
\[
P\left[\{D_g(O;\widehat\eta,\theta(g))-D_g(O;\eta^\star,\theta(g))\}^2\mid \widehat\eta\right].
\]
By main-text Assumption~7, equivalently by main-text Lemma~1 under the stated primitive conditions, this conditional variance converges to zero in probability.  Chebyshev's inequality therefore gives
\[
\sqrt n(\Pn-P)\{D_g(O;\widehat\eta,\theta(g))-D_g(O;\eta^\star,\theta(g))\}=o_P(1),
\]
after summing over the fixed number of folds; equivalently the unscaled empirical term is $o_P(n^{-1/2})$.  If clipping is used, the same argument applies to the clipped EIF map and leaves an additional deterministic or stochastic target-shift term $R_{\mathrm{clip},n}$.  Combining these pieces proves main-text equations~(13) and~(14).
\end{proof}

\subsection{Asymptotic linearity}\label{app:proof-an}

\begin{proof}
Main-text Theorem~3 gives
\[
\widehat\theta_g-\theta(g)=(\Pn-P)\phi_g^\star+R_n,
\]
where
\[
|R_n|\le C\sum_{t=0}^T\{r_{Q,t}+b_{Q,t}(\widehat\phi)\}
\{r_{\Gamma,t}+b_{\Gamma,t}(\widehat\phi)\}+o_P(n^{-1/2}).
\]
Main-text Assumption~8 makes the displayed sum $o_P(n^{-1/2})$, so $R_n=o_P(n^{-1/2})$.  Multiplying by $\sqrt n$ yields
\[
\sqrt n\{\widehat\theta_g-\theta(g)\}
=\frac{1}{\sqrt n}\sum_{i=1}^n\phi_g^\star(O_i)+o_P(1).
\]
By main-text Assumption~5, $E\{\phi_g^\star(O)\}=0$ and $E\{\phi_g^\star(O)^2\}=\sigma_g^2<\infty$.  The classical central limit theorem gives convergence in distribution to $N(0,\sigma_g^2)$.

For the variance estimator, define the out-of-fold estimated influence value
\[
\widehat\phi_i^g=\widehat Z_i^g-\widehat\theta_g,
\]
where $\widehat Z_i^g$ is the out-of-fold pseudo-outcome from main-text Algorithm~1; equivalently, for $i\in\mathcal I_k$, it is computed using $\widehat\eta^{(-k)}$ and may be written $\widehat\phi_i^{g,(-k)}$.  The same fold-conditional $L_2$ consistency argument used in the empirical-process part of main-text Theorem~3 gives
\[
\|\widehat\phi^g-\phi_g^\star\|_{L_2(P)}=o_P(1).
\]
Then
\begin{align*}
\left|\Pn\{(\widehat\phi^g)^2\}-P\{(\phi_g^\star)^2\}\right|
&\le \left|\Pn\{(\widehat\phi^g)^2-(\phi_g^\star)^2\}\right|
+\left|\Pn\{(\phi_g^\star)^2\}-P\{(\phi_g^\star)^2\}\right|.
\end{align*}
The second term is $o_P(1)$ by the law of large numbers.  For the first term, use
$|a^2-b^2|\le |a-b|^2+2|b||a-b|$ and Cauchy-Schwarz to obtain
\[
\Pn| (\widehat\phi^g)^2-(\phi_g^\star)^2 |
\le \Pn(\widehat\phi^g-\phi_g^\star)^2
+2\{\Pn(\phi_g^\star)^2\}^{1/2}
\{\Pn(\widehat\phi^g-\phi_g^\star)^2\}^{1/2}=o_P(1).
\]
Thus $\widehat\sigma_g^2=\Pn(\widehat\phi_i^g)^2$ is consistent for $\sigma_g^2$, and the Wald interval has asymptotic coverage $1-\alpha$.
\end{proof}

\subsection{Clipping and overlap}\label{app:proof-clip}

\begin{proof}
For a finite threshold $\tau$, define the clipped estimating function by replacing $\mathcal H_t^g$ in main-text equation~(6) with $\mathcal H_{t,\tau}^g=\min(\mathcal H_t^g,\tau)$.  This construction gives a bounded estimating equation, but it is not automatically the canonical gradient of a causal parameter.  The reason is that $\min(\mathcal H_t^g,\tau)$ is a history-dependent transformation of the likelihood ratio.  Unless a corresponding weighted intervention distribution is specified, the clipped equation is simply a regularized estimating equation.

When strict positivity holds and $\tau_n\to\infty$, the clipped and unclipped estimating functions are asymptotically equivalent.  Let
\[
D_{\tau_n}(O)-D(O)=\sum_{t=0}^T(\mathcal H_{t,\tau_n}^{g,\star}-\mathcal H_t^{g,\star})
\{Q_{t+1}^{g,\star}(H_{t+1})-Q_t^{g,\star}(H_t)\}.
\]
If $E|D_{\tau_n}(O)-D(O)|=o(n^{-1/2})$, or the corresponding $L_2$ condition holds for variance estimation, clipping contributes only $o_P(n^{-1/2})$ to the asymptotic expansion.  This condition is satisfied, for example, when the clever covariates have a sufficiently light tail and $\tau_n$ grows fast enough.  Thus clipping is justified as asymptotically negligible stabilization for $\theta(g)$ under strict positivity.

If practical positivity is intrinsically poor, the causal target can instead be defined with an explicit weighting function.  For the overlap target
\[
\theta_\omega(g)=\frac{N_g}{D},\qquad N_g=E\{\omega(H_0)Y^g\},\quad D=E\{\omega(H_0)\},
\]
with fixed nonnegative $\omega$, the numerator EIF is the usual longitudinal EIF multiplied by $\omega(H_0)$ plus the baseline contribution induced by perturbing the distribution of $H_0$.  Equivalently, write $\phi_{N_g}$ for the EIF of $N_g$ and $\phi_D=\omega(H_0)-D$.  The quotient rule gives
\[
\phi_{\omega,g}=D^{-1}\phi_{N_g}-N_gD^{-2}\phi_D.
\]
This is a genuine causal overlap estimand because $\omega$ is part of the estimand.  It is distinct from post hoc truncation of $\mathcal H_t^g$.  The theorem follows.
\end{proof}

\subsection{Robust aggregation}\label{app:proof-robust}

\begin{proof}
Condition first on the outer training folds.  The oracle pseudo-outcomes $Z_i^g=\widetilde Y_{i0}^g(\eta^\star)$ are i.i.d. over evaluation observations with mean $\theta(g)$ and variance bounded by $v$.  Let $\psi$ be the standard Catoni influence function satisfying
\[
-\log(1-x+x^2/2)\le \psi(x)\le \log(1+x+x^2/2).
\]
For a fixed candidate $m$, exponential Markov inequalities applied to $\sum_i\psi\{\alpha(Z_i^g-m)\}$ bound the two tail probabilities of the Catoni root.  Choosing $\alpha\asymp \sqrt{\log(1/\delta)/(nv)}$ yields
\[
P\left(|\widehat\theta_{C}^{\mathrm{or}}-\theta(g)|>C\sqrt{v\log(1/\delta)/n}\mid\text{training folds}\right)\le \delta
\]
for a universal constant $C$, up to the usual small constants in Catoni's theorem.  If $v$ is unknown, it is estimated on an auxiliary split or replaced by a conservative upper bound; the scale parameter is therefore part of the estimator specification.

Now compare feasible and oracle Catoni roots.  Let
\[
\Psi_n(m;z)=n^{-1}\sum_{i=1}^n\psi\{\alpha(z_i-m)\}.
\]
Because $\psi$ is Lipschitz, uniformly over $m$ in a neighborhood of $\theta(g)$,
\[
|\Psi_n(m;\widehat Z)-\Psi_n(m;Z)|
\le \alpha n^{-1}\sum_i|\widehat Z_i^g-Z_i^g|=O_P(\alpha\delta_n).
\]
The derivative of the oracle estimating function satisfies
\[
-\partial_m\Psi_n(m;Z)=\alpha n^{-1}\sum_i
\psi'\{\alpha(Z_i^g-m)\}.
\]
Under the local concentration supplied by the oracle Catoni argument, this derivative is bounded below by $c\alpha$ with probability tending to one in a neighborhood of $\theta(g)$.  A one-dimensional Z-estimation expansion therefore divides the $O_P(\alpha\delta_n)$ perturbation by an order-$\alpha$ derivative and gives
\[
|\widehat\theta_C-\widehat\theta_C^{\mathrm{or}}|=O_P(\delta_n).
\]
Combining this transfer bound with the oracle Catoni inequality gives
\[
|\widehat\theta_C-\theta(g)|
=O_P\{\sqrt{v\log(1/\delta)/n}\}+O_P(\delta_n).
\]
The quantity $\delta_n=n^{-1}\sum_i|\widehat Z_i^g-Z_i^g|$ is an additional feasible-pseudo-outcome stability condition for robust aggregation.  It is not a direct conclusion of the scalar mean-estimator expansion in main-text Theorem~3.  The condition $\delta_n=o_P(1)$ yields consistency of the robust aggregate; retaining the oracle Catoni $n^{-1/2}$ concentration order would require $\delta_n=o_P(n^{-1/2})$ or a comparably sharp transfer bound.  Under bounded or clipped clever covariates, stable backward recursion, and $L_1$ or stronger $L_2$ convergence of the nuisance functions, $\delta_n=o_P(1)$ can be verified by applying the same cross-fitted pseudo-outcome perturbation bounds used in the product-remainder proof.  Fold dependence is harmless because all conditioning is on $\mathcal T_k$ and the final evaluation observations are independent within each fold; summing over a fixed number of folds preserves the order.

For median-of-means, form blocks independently of the evaluation outcomes, for example by a random partition within the outer held-out folds.  Oracle block means have the usual median concentration: under finite variance, Chebyshev's inequality makes a majority of block means accurate with high probability; under a finite $(1+\kappa)$ moment the block means converge at the corresponding polynomial rate.  If the average feasible-oracle discrepancy is $\delta_n$, all but a vanishing fraction of block means are shifted by at most $O_P(\delta_n)$, so the median is shifted by the same order.  This proves the robust aggregation statement.
\end{proof}

\subsection{Basis-encoder rates}\label{app:proof-sieve}

\begin{proof}
Let
\[
X_t(u)=\sum_{k=1}^{\infty}\alpha_{tk}\psi_k(u)
\]
be the Karhunen-Loeve expansion, with eigenvalues $\lambda_k\lesssim k^{-2a}$.  A $K_n$-term basis encoder retains the first $K_n$ scores and estimates them from the sparse noisy point cloud.  The representation error has two components.

The first component is truncation.  Under the smoothness assumption in the corollary, the true nuisance functions can be approximated by functions of the first $K_n$ scores with error $O(K_n^{-s})$ in $L_2(P)$.  This covers, for example, Holder or Sobolev smooth nuisance maps whose dependence on high-frequency scores decays with the eigenvalues.  The second component is score-estimation error.  If the estimated score vector has mean integrated squared error $\epsilon_{n,K}^2$, Lipschitz or Holder continuity of the nuisance maps in the score vector gives an additional approximation term $O(\epsilon_{n,K})$.

The quantity $\epsilon_{n,K}$ is intentionally defined at the level of the observed-history encoder rather than only through an idealized dense-curve design.  For noninformative sparse sampling, it may be verified by standard sparse-FPCA or spline-smoothing arguments when the expected number of measurements per visit is bounded away from zero and the measurement noise has bounded variance.  When the sampling process is informative, as in Scenario F of the simulation, the measurement count, locations, and marker identities are part of the observed history.  The score-estimation error should then be read as the conditional error of the chosen encoder given this observed sampling design.  In that case the theorem does not require the sampling design to be ignorable; it requires the encoder class to approximate the nuisance functions of the observed point cloud and its sampling pattern.

Consequently, for both the outcome regression and mechanism nuisance classes,
\[
b_{Q,t}(\phi)\lesssim K_n^{-s}+\epsilon_{n,K},\qquad
b_{\Gamma,t}(\phi)\lesssim K_n^{-s}+\epsilon_{n,K}.
\]
Combining these bounds with the stochastic nuisance errors $r_{Q,t}$ and $r_{\Gamma,t}$ gives
\[
\sum_t\{r_{Q,t}+b_{Q,t}(\phi)\}\{r_{\Gamma,t}+b_{\Gamma,t}(\phi)\}
\lesssim
\sum_t\{r_{Q,t}+K_n^{-s}+\epsilon_{n,K}\}
\{r_{\Gamma,t}+K_n^{-s}+\epsilon_{n,K}\}.
\]
The condition displayed in the corollary makes the right-hand side $o_P(n^{-1/2})$, which is exactly main-text Assumption~8.  The result follows from main-text Theorem~4.  Deep Set Transformer and neural-ODE encoders are not required for this verification result; the corollary supplies a concrete sieve route showing that the high-level representation condition is attainable.
\end{proof}

\section{Stochastic-Regime Extension}\label{app:stochastic-regimes}

The main article writes the theorem sequence for deterministic static or dynamic regimes.  This section records the parallel notation for stochastic interventions.  Let $q=\{q_t(\cdot\mid H_t)\}_{t=0}^T$ be a user-specified intervention density or probability mass function that is absolutely continuous with respect to the observed treatment law on the support of $q$.  Define
\[
\omega_t^q(O)=
\frac{q_t(A_t\mid H_t)R_t}
{\pi_t(A_t\mid H_t)\rho_t(H_t,A_t)},\qquad
\mathcal H_t^q(O)=\prod_{s=0}^t\omega_s^q(O),
\]
where $\pi_t(a\mid h)=P(A_t=a\mid H_t=h,\bar R_{t-1}=1)$ and $\rho_t(h,a)=P(R_t=1\mid H_t=h,A_t=a,\bar R_{t-1}=1)$.  The stochastic-regime recursion replaces conditioning on $A_t=g_t(H_t)$ by integration over the intervention:
\[
Q_t^q(h)=\int E\{Q_{t+1}^q(H_{t+1})\mid H_t=h,A_t=a,R_t=1,\bar R_{t-1}=1\}\,q_t(da\mid h),
\]
with $Q_{T+1}^q=Y$ and $\theta(q)=E\{Q_0^q(H_0)\}$.  Under the same consistency, sequential exchangeability, observed-history, and positivity conditions, with bounded density ratios $q_t/\pi_t$, the observed-data EIF is
\[
\phi_q(O)=
\sum_{t=0}^T \mathcal H_t^q(O)
\{Q_{t+1}^q(H_{t+1})-m_t^q(H_t,A_t)\}
+Q_0^q(H_0)-\theta(q),
\]
where
\[
m_t^q(H_t,A_t)=E\{Q_{t+1}^q(H_{t+1})\mid H_t,A_t,R_t=1,\bar R_{t-1}=1\}.
\]
Equivalently, one may write the recursive pseudo-outcome with the increment $\omega_t^q$ and the conditional regression $m_t^q$.  The same cross-fitting, representation approximation, and product-remainder arguments apply after replacing the deterministic treatment-indicator mechanism by the stochastic density ratio.  We keep the main text deterministic because the simulations and VitalDB illustration use deterministic dynamic policies and a point-treatment analogue, respectively.

\section{Simulation Design Details}\label{app:simulation-details}

The simulation code implements a longitudinal functional-history design with $T+1=6$ decision times.  Each subject has baseline covariates $B=(B_1,B_2,B_3)$, latent functional scores $\alpha_{itk}$, scalar time-varying covariates $L_{it}$, irregular point-cloud summaries, binary treatments, and a continuous terminal outcome.  The latent scores follow an autoregressive evolution with coefficient 0.55 and geometrically decaying marginal scale.  Scenario G increases the number of latent scores to 12 so that confounding-relevant low-variance components are harder to capture by simple summaries.

At each visit, a nonlinear latent functional feature combines the first scores and, in Scenarios B and G, interactions and nonlinear score terms.  The scalar severity, lactate-like marker, and additional physiological markers depend on the latent feature, baseline covariates, lagged treatment, and noise.  Irregular measurement counts are generated from a Poisson model whose mean depends on current severity; in Scenario F, the measurement intensity is made more strongly confounding-relevant.  The observed point-cloud representation records noisy mean, range, standard deviation, count, relative count, and severity-weighted measurement intensity.  The simulation therefore targets an observed-history estimand rather than an unrestricted latent full-function estimand.

Observed treatment assignment is binary.  Its logit depends on current severity, the lactate-like marker, lagged treatment, baseline covariates, observed functional summaries, and, in Scenarios B, F, and G, additional nonlinear or sampling-intensity terms.  Scenario D strengthens the treatment logit to create near-positivity.  No censoring process is simulated in the main Monte Carlo study; censoring is included in the estimator and theory but is not stressed empirically in these tables.  This choice keeps the simulation focused on functional representation, treatment overlap, and heavy-tailed pseudo-outcomes.

The two target policies are an early policy, which treats when current severity exceeds a threshold or the lactate-like marker is high, and a delayed policy, which treats after persistent severity elevation or very high lactate.  The reported estimand is the contrast $\theta(g_{\rm early})-\theta(g_{\rm delayed})$.  Population truths are computed from independent Monte Carlo intervention samples of size 60,000.  The main experiment uses $n=400$, and the sample-size experiment uses $n\in\{200,400,800\}$; each simulation cell uses 150 replications.  For a nominal 95\% coverage probability, the Monte Carlo standard error is approximately $\sqrt{0.95(0.05)/150}=0.018$; bias MCSEs are recorded in the generated CSV summaries.

The reported sequence-style baselines use method-specific nonlinear feature maps and ridge or tree-strength regularization so that the full experiment is reproducible on CPU.  They are intended as objective-level baselines for counterfactual prediction, not as exact reproductions of GPU-trained architectures.  The latent-score SDR diagnostic receives the true latent scores but still estimates finite-sample outcome and treatment nuisance functions and therefore should not be interpreted as an oracle upper bound.

\section{Additional VitalDB Diagnostics}\label{app:vitaldb-diagnostics}

The main article reports the VitalDB cohort, point-cloud examples, representation ablation, overlap/calibration figure, and trimming sensitivity.  This section gives the remaining diagnostics used to audit the point-treatment illustration.

\begin{figure}[H]
\centering
\includegraphics[width=0.85\textwidth]{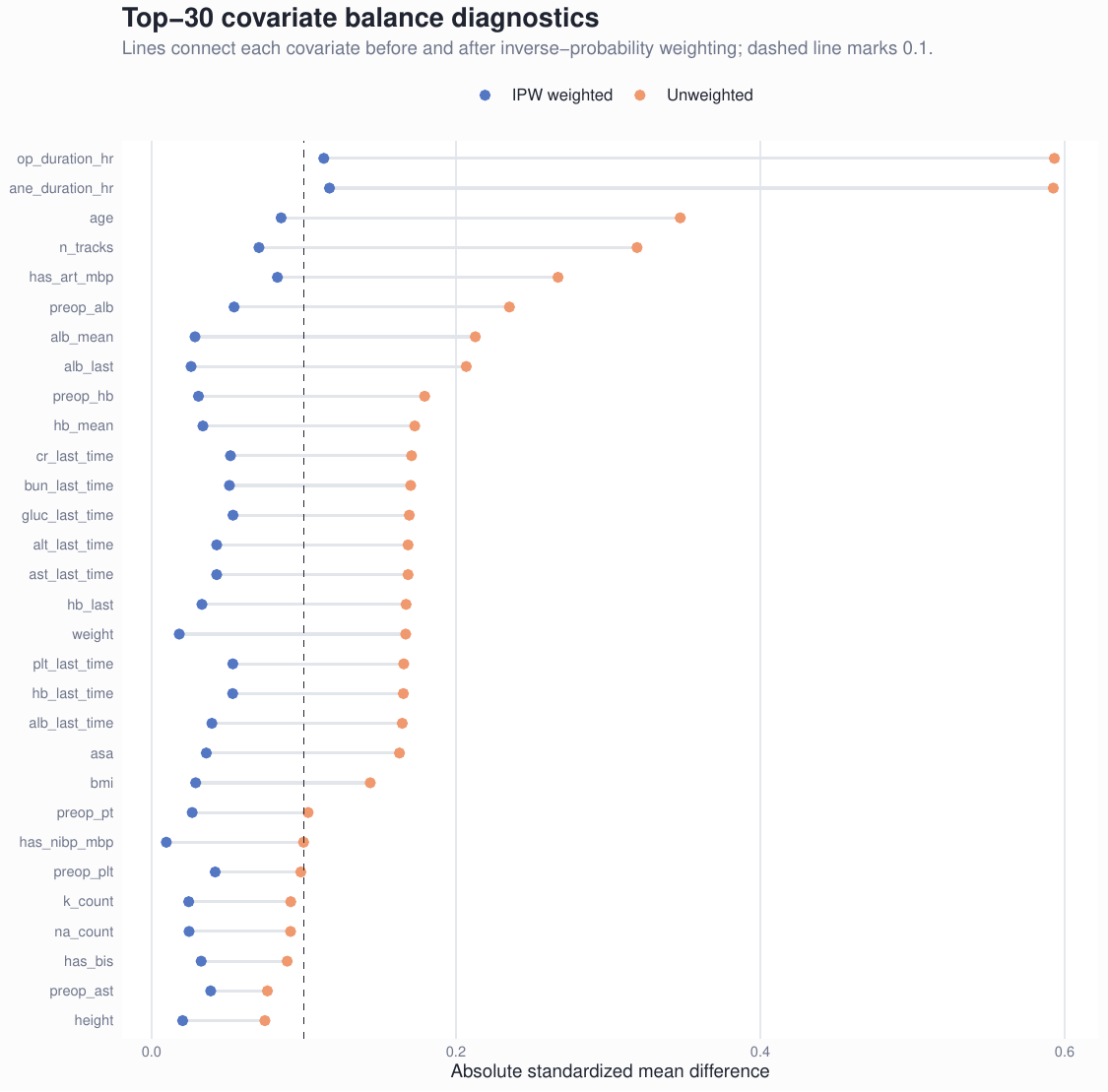}
\caption{Top-30 covariate balance diagnostics for the VitalDB analysis.  The dashed line marks 0.1.}
\label{fig:app-vitaldb-love}
\end{figure}

\begin{table}[H]
\centering
\caption{Top 30 standardized mean differences before and after weighting for the functional DR-FRL analysis.}
\label{tab:vitaldb-top30-balance}
\scalebox{0.9}{
\begin{tabular}{p{0.45\textwidth}rr}
\toprule
Covariate & SMD unweighted & SMD weighted \\
\midrule
op\_duration\_hr & 0.593 & 0.113 \\
ane\_duration\_hr & 0.592 & 0.117 \\
age & 0.347 & 0.085 \\
n\_tracks & 0.319 & 0.071 \\
has\_art\_mbp & 0.267 & 0.083 \\
preop\_alb & -0.235 & -0.054 \\
alb\_mean & -0.213 & -0.029 \\
alb\_last & -0.207 & -0.026 \\
preop\_hb & -0.179 & -0.031 \\
hb\_mean & -0.173 & -0.034 \\
cr\_last\_time & -0.171 & -0.052 \\
bun\_last\_time & -0.170 & -0.051 \\
gluc\_last\_time & -0.169 & -0.053 \\
alt\_last\_time & -0.168 & -0.043 \\
ast\_last\_time & -0.168 & -0.043 \\
hb\_last & -0.167 & -0.033 \\
weight & -0.167 & -0.018 \\
plt\_last\_time & -0.166 & -0.053 \\
hb\_last\_time & -0.165 & -0.053 \\
alb\_last\_time & -0.165 & -0.040 \\
asa & 0.163 & 0.036 \\
bmi & -0.144 & -0.029 \\
preop\_pt & -0.103 & -0.027 \\
has\_nibp\_mbp & -0.100 & -0.010 \\
preop\_plt & -0.098 & -0.042 \\
k\_count & 0.091 & 0.024 \\
na\_count & 0.091 & 0.025 \\
has\_bis & 0.089 & 0.033 \\
preop\_ast & 0.076 & 0.039 \\
height & -0.074 & 0.020 \\
\bottomrule
\end{tabular}}
\end{table}

\begin{table}[H]
\centering
\caption{Preoperative laboratory point-cloud missingness and measurement intensity.}
\label{tab:vitaldb-missingness}
\begin{tabular}{lrrr}
\toprule
Variable & Observed cases & Observed \% & Median points among observed \\
\midrule
hb & 3909 & 61.2 & 1 \\
plt & 3898 & 61.0 & 1 \\
cr & 3746 & 58.6 & 1 \\
bun & 3745 & 58.6 & 1 \\
alb & 3811 & 59.7 & 1 \\
gluc & 3286 & 51.4 & 1 \\
na & 4543 & 71.1 & 1 \\
k & 4544 & 71.1 & 1 \\
ast & 3812 & 59.7 & 1 \\
alt & 3812 & 59.7 & 1 \\
\bottomrule
\end{tabular}
\end{table}

Table~\ref{tab:vitaldb-missingness} summarizes missingness and measurement intensity for the functional laboratory point clouds.  The analysis treats marker-specific missingness and measurement counts as part of the observed history by including count and kernel-mass features.

\begin{table}[H]
\centering
\caption{Exposure and outcome sensitivity analyses using the same functional representation.}
\label{tab:vitaldb-sensitivity}
\resizebox{\linewidth}{!}{
\begin{tabular}{p{0.18\textwidth}p{0.19\textwidth}rrrrr}
\toprule
Exposure & Outcome & Prevalence A & Prevalence Y & RD & SE & 95\% CI \\
\midrule
Any vasopressor & ICU admission & 0.543 & 0.188 & 0.010 & 0.008 & [-0.004, 0.025] \\
Any vasopressor & ICU stay $>$ 1 day & 0.543 & 0.061 & 0.003 & 0.008 & [-0.013, 0.020] \\
Any vasopressor & Composite adverse & 0.543 & 0.066 & 0.001 & 0.013 & [-0.025, 0.027] \\
Phenylephrine & ICU admission & 0.132 & 0.188 & 0.058 & 0.021 & [0.017, 0.099] \\
Phenylephrine & ICU stay $>$ 1 day & 0.132 & 0.061 & -0.007 & 0.016 & [-0.038, 0.024] \\
Phenylephrine & Composite adverse & 0.132 & 0.066 & 0.011 & 0.018 & [-0.025, 0.047] \\
Ephedrine & ICU admission & 0.503 & 0.188 & 0.007 & 0.019 & [-0.030, 0.044] \\
Ephedrine & ICU stay $>$ 1 day & 0.503 & 0.061 & -0.039 & 0.017 & [-0.072, -0.006] \\
Ephedrine & Composite adverse & 0.503 & 0.066 & -0.018 & 0.012 & [-0.041, 0.006] \\
High-dose vasopressor & ICU admission & 0.169 & 0.188 & 0.061 & 0.020 & [0.022, 0.101] \\
High-dose vasopressor & ICU stay $>$ 1 day & 0.169 & 0.061 & 0.028 & 0.019 & [-0.010, 0.065] \\
High-dose vasopressor & Composite adverse & 0.169 & 0.066 & 0.035 & 0.017 & [0.002, 0.069] \\
\bottomrule
\end{tabular}}
\end{table}

\begin{figure}[H]
\centering
\includegraphics[width=0.92\textwidth]{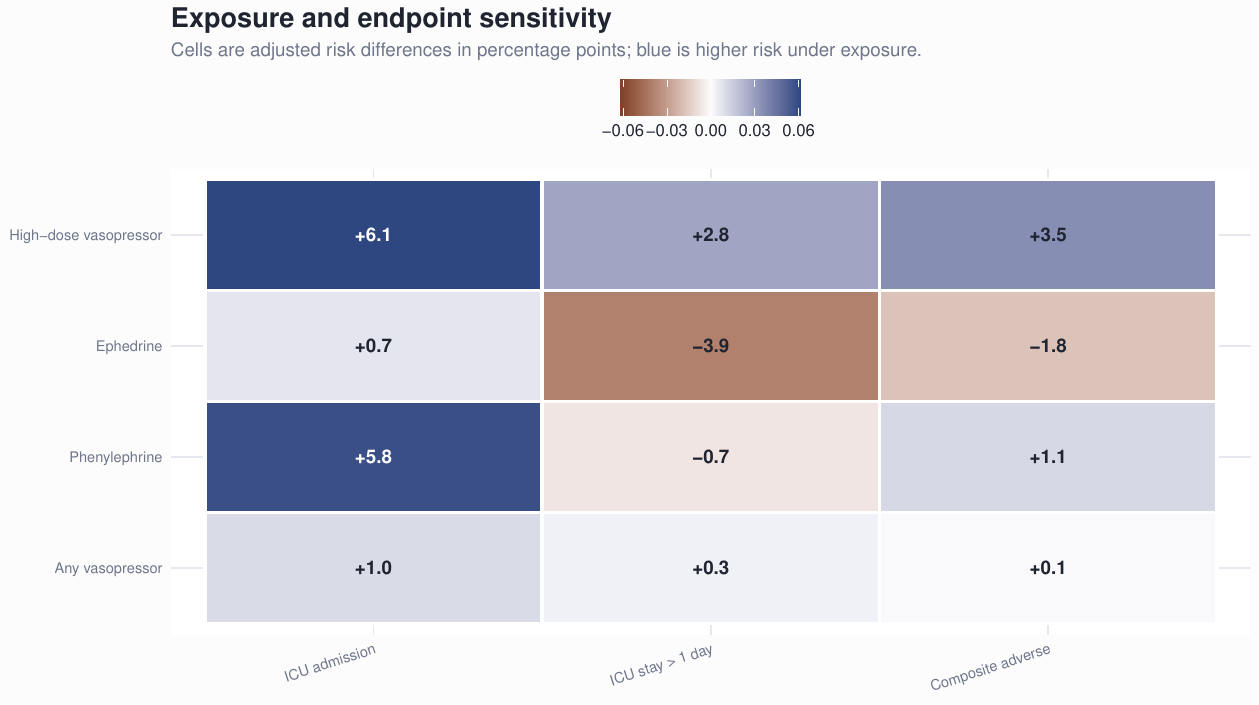}
\caption{Exposure and outcome sensitivity analysis.  Cells are adjusted risk differences from functional-representation AIPW fits or lightweight sensitivity fits using the same point-cloud representation.}
\label{fig:app-vitaldb-sensitivity-heatmap}
\end{figure}

Phenylephrine and high-dose vasopressor exposure show larger adjusted ICU-admission risk differences than the pooled any-vasopressor definition.  Ephedrine is less stable across endpoints, and the composite postoperative endpoint is sparse.  These results support treating the VitalDB analysis as a feasibility and auditing example rather than as definitive clinical evidence.

\begin{figure}[H]
\centering
\includegraphics[width=0.88\textwidth]{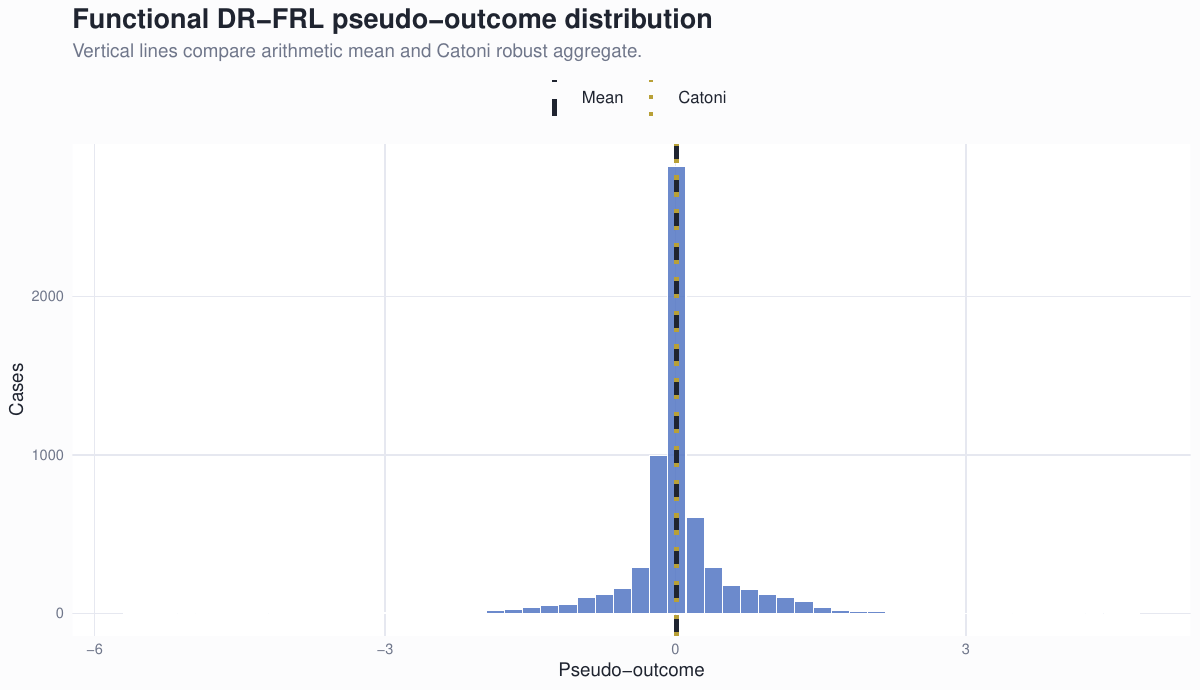}
\caption{Pseudo-outcome distribution for the functional DR-FRL estimator.  The vertical dashed and dotted lines mark the arithmetic mean and Catoni aggregate, respectively.}
\label{fig:app-vitaldb-pseudo-tail}
\end{figure}

The AIPW pseudo-outcome has moderate tails even for a binary endpoint because inverse-probability correction amplifies residual variation.  The Catoni aggregate changes the primary estimate only slightly, consistent with the main text's conclusion that robustification is most important when pseudo-outcome tails or overlap diagnostics indicate instability.

\end{document}